\pdfoutput=1 
\documentclass{article}
\usepackage{ijcai22}

\usepackage{times}
\usepackage{soul}
\usepackage{url}
\usepackage[hidelinks]{hyperref}
\usepackage[utf8]{inputenc}
\usepackage[small]{caption}
\usepackage{graphicx}
\usepackage{amsmath}
\usepackage{amsthm}
\usepackage{booktabs}
\usepackage{algorithm}
\usepackage{algorithmic}
\usepackage{amsmath}
\usepackage{amssymb}
\usepackage{mathtools}
\usepackage{amsthm}
\usepackage{natbib}

\usepackage[capitalize,noabbrev]{cleveref}

\usepackage{mathtools}
\usepackage{todonotes}
\usepackage{amssymb}

\newcommand{\norm}[1]{\left\lVert#1\right\rVert}

\newcommand{\Dcal}[0]{\mathcal{D}}
\newcommand{\Lcal}[0]{\mathcal{L}}

\newcommand{\Ocal}[0]{\mathcal{O}}

\newcommand{\expect}{\operatorname{\mathbb{E}}}
\newcommand{\Var}{\operatorname{Var}}
\newcommand{\Cov}{\operatorname{Cov}}
\newcommand{\VAR}[1]{\Var\left[#1\right]}
\newcommand{\VARR}[2]{\Var_{#1}\left[#2\right]}
\newcommand{\COV}[2]{\Cov\left[#1 , #2\right]}
\newcommand{\E}[1]{\expect\left[#1\right]}
\newcommand{\EE}[2]{\expect_{#1}\left[#2\right]}

\newcommand{\inner}[2]{\langle #1, #2 \rangle}

\usepackage{amsmath,amsfonts,bm}

\def\eqref#1{equation~\ref{#1}}

\def\1{\bm{1}}

\def\vtheta{{\bm{\theta}}}
\def\vxi{{\bm{\xi}}}
\def\va{{\bm{a}}}
\def\vb{{\bm{b}}}

\def\vd{{\bm{d}}}

\def\vh{{\bm{h}}}

\def\vx{{\bm{x}}}
\def\vy{{\bm{y}}}

\DeclareMathAlphabet{\mathsfit}{\encodingdefault}{\sfdefault}{m}{sl}
\SetMathAlphabet{\mathsfit}{bold}{\encodingdefault}{\sfdefault}{bx}{n}

\newtheorem{theorem}{Theorem}
\newtheorem{definition}{Definition}
\newtheorem{assumption}{Assumption}
\newtheorem{proposition}{Proposition}
\newtheorem{lemma}{Lemma}
\newtheorem{corollary}{Corollary}

\title{A General Theory for Client Sampling in Federated Learning}

\author{
Yann Fraboni$^{1, 2}$
\and
Richard Vidal$^2$\and
Laetitia Kameni$^{2}$\And
Marco Lorenzi$^1$
\affiliations
$^1$Universit\'e C\^{o}te d’Azur, Inria Sophia Antipolis,
	Epione Research Group, France\\
$^2$Accenture Labs, Sophia Antipolis, France\\
\thanks{Accepted to the International Workshop on Trustworthy Federated Learning in Conjunction with IJCAI 2022.}
}

\begin{document}

\maketitle

\begin{abstract}

While client sampling is a central operation of current state-of-the-art federated learning (FL) approaches, the impact of this procedure on the convergence and speed of FL remains under-investigated. In this work, we provide a general theoretical framework to quantify the impact of a client sampling scheme and of the clients heterogeneity on the federated optimization. First, we provide a unified theoretical ground for previously reported sampling schemes experimental results on the relationship between FL convergence and the variance of the aggregation weights. Second, we prove for the first time that the quality of FL convergence is also impacted by the resulting \emph{covariance} between aggregation weights. Our theory is general, and is here applied to Multinomial Distribution (MD) and Uniform sampling, two default unbiased client sampling schemes of FL, and demonstrated through a series of experiments in non-iid and unbalanced scenarios. Our results suggest that MD sampling should be used as default sampling scheme, due to the resilience to the changes in data ratio during the learning process, while Uniform sampling is superior only in the special case when clients have the same amount of data.
	

\end{abstract}

\section{Introduction}

Federated Learning (FL) has gained popularity in the last years as it enables different clients to jointly learn a global model without sharing their respective data. 
Among the different FL approaches, federated averaging (\textsc{FedAvg}) has emerged as the most popular optimization scheme \citep{FedAvg}. An optimization round of \textsc{FedAvg} requires data owners, also called clients, to receive from the server the current global model which they update on a fixed amount of Stochastic Gradient Descent (SGD) steps before sending it back to the server. The new global model is then created as the weighted average of the client updates, according to their data ratio.
FL specializes the classical problem of distributed learning (DL), to account for the private nature of clients information (i.e. data and surrogate features), and for the potential data and hardware heterogeneity across clients, which is generally unknown to the server.

In FL optimization, \textsc{FedAvg} was first proven to converge experimentally \citep{FedAvg}, before theoretical guarantees were provided for any non-iid federated dataset \citep{FedNova, SCAFFOLD, haddadpour2019convergence, Khaled2020}. A drawback of naive implementations of \textsc{FedAvg} consists in requiring the participation of all the clients to every optimization round. As a consequence, the efficiency of the optimization is limited by the communication speed of the slowest client, as well as by the server communication capabilities.  To mitigate this issue, the original \textsc{FedAvg} algorithm already contemplated the possibility of considering a random subset of $m$ clients at each FL round. It has been subsequently shown that, to ensure the convergence of FL to its optimum, clients must be sampled such that in expectation the resulting global model is identical to the one obtained when considering all the clients \citep{FedNova, PowerOfChoice}. Clients sampling schemes compliant with this requirement are thus called \textit{unbiased}. Due to its simplicity and flexibility, the current default unbiased sampling scheme consists in sampling $m$ clients according to a Multinomial Distribution (MD), where the sampling probability  depends on the respective data ratio \citep{FedProx, FedNova, OnTheConvergence, haddadpour2019convergence, Li2020Fair, Wang2018Cooperative, ClusteredSampling}. Nevertheless, when clients have identical amount of data, clients can also be sampled uniformly without replacement \citep{OnTheConvergence, SCAFFOLD, reddi2021adaptive, DynamicFL}. In this case, Uniform sampling has been experimentally shown to yield better results than MD sampling \citep{OnTheConvergence}.

Previous works proposed unbiased sampling strategies alternative to MD and Uniform sampling with the aim of improving FL convergence. 
In \cite{ClusteredSampling}, MD sampling was extended to account for clusters of clients with similar data characteristics, while in \cite{Richtarik_optimal_sampling}, clients sampling probabilities are defined depending on the Euclidean norm of the clients local work.  
While these works are based on the definition and analysis of specific sampling procedures, aimed at satisfying a given FL criterion, there is currently a need for a general theoretical framework to elucidate the impact of client sampling on FL convergence.  

The main contribution of this work consists in deriving a general theoretical framework for FL optimization allowing to clearly quantify the impact of client sampling on the global model update at any FL round. This contribution has important theoretical and practical implications. First, we  demonstrate the dependence of FL convergence on the variance of the aggregation weights. Second, we prove for the first time that the convergence speed is also impacted through sampling by the resulting \emph{covariance} between aggregation weights. 
From a practical point of view, we establish both theoretically and experimentally that client sampling schemes based on aggregation weights with sum different than 1 are less efficient. We also prove that MD sampling is outperformed by Uniform sampling only when clients have identical data ratio. Finally, we show that the comparison between different client sampling schemes is appropriate only when considering a small number of clients. 
Our theory ultimately shows that MD sampling should be used as default sampling scheme, due to the favorable statistical properties and to the resilience to FL applications with varying data ratio and heterogeneity.

Our work is structured as follows. 
In Section \ref{sec:related}, we provide formal definitions for FL, unbiased client sampling, and for the server aggregation scheme. 
In Section \ref{sec:convergence}, we introduce our convergence guarantees (Theorem \ref{theo:convergence_paper}) relating the convergence of FL to the aggregation weight variance of the client sampling scheme.
Consistently with our theory, in Section \ref{sec:experiments}, we experimentally demonstrate the importance of the clients aggregation weights variance and covariance on the convergence speed, and conclude by recommending Uniform sampling for FL applications with identical client ratio, and MD sampling otherwise.

\section{Background}
\label{sec:related}

Before investigating in Section \ref{sec:convergence} the impact of client sampling  on FL convergence, we recapitulate in Section \ref{sec:related} the current theory behind FL aggregation schemes for clients local updates. We then introduce a formalization for \textit{unbiased} client sampling.

\subsection{Aggregating clients local updates}

In FL, we consider a set $I$ of $n$ clients each respectively owning a dataset $\Dcal_i$ composed of $n_i$ samples. 
FL aims at optimizing the average of each clients local loss function weighted by $p_i$ such that $\sum_{i=1}^{n} p_i = 1$, i.e.
\begin{equation}
\label{eq:global_loss}
	\Lcal(\theta) = \sum_{i=1}^n p_i \Lcal_i(\theta),
\end{equation}
where $\theta$ represents the model parameters. The weight $p_i$ can be interpreted as the importance given by the server to client $i$ in the federated optimization problem. While any combination of $\{p_i\}$ is possible, we note that in practice, either (a) every device has equal importance, i.e. $p_i = 1/n$, or (b) every data point is equally important, i.e. $p_i = n_i /M$ with $M=\sum_{i=1}^n n_i$. Unless stated otherwise, in the rest of this work, we consider to be in case (b), i.e. $\exists i, \ p_i \neq 1/n$.

In this setting, to estimate a global model across clients, \textsc{FedAvg} \citep{FedAvg} is an iterative training strategy based on the aggregation of local model parameters. At each iteration step $t$, the server sends the current global model parameters $\theta^t$ to the clients. Each client updates the respective model by minimizing the local cost function $\Lcal_i(\theta)$ through a fixed amount $K$ of SGD steps initialized with $\theta^t$. Subsequently each client returns the updated local parameters ${\theta}_i^{t+1}$ to the server. The global model parameters $\theta^{t+1}$ at the iteration step $t+1$ are then estimated as a weighted average:
\begin{equation}
\label{eq:FedAvg_server_aggregation}
\theta^{t+1}=\sum_{i=1}^np_i{\theta}_i^{t+1}.
\end{equation}

To alleviate the clients workload and reduce the amount of overall communications, the server often considers $m \le n$ clients at every iteration. In heterogeneous datasets containing many workers, the percentage of sampled clients $m/n$ can be small, and thus induce important variability in the new global model, as each FL optimization step necessarily leads to an improvement on the $m$ sampled clients to the detriment of the non-sampled ones. To solve this issue, \citet{reddi2021adaptive, SCAFFOLD, Wang2020SlowMo} propose considering an additional learning rate $\eta_g$ to better account for the clients update at a given iteration. We denote by $\omega_i(S_t)$ the stochastic aggregation weight of client $i$ given the subset of sampled clients $S_t$ at iteration $t$ .
The server aggregation scheme can be written as:
\begin{equation}
\label{eq:aggreg_SCAFFOLD_general}
\theta^{t+1}
= \theta^t + \eta_g \sum_{i =1}^n \omega_i(S_t)(\theta_i^{t+1} - \theta^t)
.
\end{equation}

\subsection{Unbiased data agnostic client samplings}

\begin{table*}[ht]
	\caption{Synthesis of statistical properties of different sampling schemes. }
	\label{table:sampling}

	\begin{center}
	\begin{small}
	\begin{sc}
	\begin{tabular}{llll}
		\toprule
		Sampling  
		& $\VAR{\omega_i(S_t)}$    
		& $\alpha$    
		& $\VAR{\sum_{i=1}^{n}\omega_i(S_t)}$ \\
		\midrule
		
		Full participation  
		& $= 0$   
		& $= 0$    
		& $=0$     \\
		
		MD  
		& $ = - \frac{1}{m} p_i^2 + \frac{1}{m} p_i$
		& $= 1/m$    
		& $=0$     \\
		
		Uniform     
		& $ = \left(\frac{n}{m} - 1\right) p_i^2$
		& $ = \frac{n - m}{m(n-1)}$     
		& $= \frac{n-m}{m (n-1)} [n \sum_{i=1}^{n}p_i^2 -1]$ \\
		
		\bottomrule
	\end{tabular}
	\end{sc}
	\end{small}
	\end{center}
	\vskip -0.1in
\end{table*}

While \textsc{FedAvg}  was originally based on the uniform sampling of clients \citep{FedAvg}, this scheme has been proven to be biased and converge to a suboptimal minima of problem (\ref{eq:global_loss}) \citep{FedNova, PowerOfChoice, OnTheConvergence}. This was the motivation for \cite{OnTheConvergence} to introduce the notion of \textit{unbiasedness}, where clients are considered in expectation subject to their importance $p_i$, according to Definition \ref{def:sampling} below.
Unbiased sampling guarantees the optimization of the original FL cost function, while minimizing the number of active clients per FL round.
We note that unbiased sampling is not necessarily related to the clients distribution, as this would require to know beforehand the specificity of the clients' datasets. 

Unbiased sampling methods \citep{FedProx, OnTheConvergence, ClusteredSampling} are currently among the standard approaches to FL, as opposed to \textit{biased} approaches, known to over- or under-represent clients and lead to suboptimal convergence properties \citep{FedAvg, sampling_mobile_edge, optimal_user_selection, PowerOfChoice}, or to methods requiring additional computation work from clients \citep{Richtarik_optimal_sampling}. 

\begin{definition}[Unbiased Sampling]\label{def:sampling} 
	A client sampling scheme is said unbiased if the expected value of the client aggregation is equal to the global deterministic aggregation obtained when considering all the clients, i.e.
	\begin{equation}
	\EE{S_t}{\sum_{i=1}^n w_i(S_t) \theta_i^t}
	\coloneqq
	\sum_{i =1 }^n p_i\theta_i^t ,
	\end{equation}
	
	where $w_j(S_t)$ is the aggregation weight of client $j$ for subset of clients $S_t$.
\end{definition}

The sampling distribution uniquely defines the statistical properties of stochastic weights. In this setting, unbiased sampling guarantees the equivalence between deterministic and stochastic weights in expectation.
%
Unbiased schemes of primary importance in FL are MD and Uniform sampling, for which we can derive a close form formula for the aggregation weights :

\textbf{MD sampling}. This scheme considers $l_1, ..., l_m$ to be the $m$ iid sampled clients from a Multinomial Distribution with support on $\{1, ..., m\}$ satisfying $\mathbb{P}(l_k=i) = p_i$ \citep{FedNova, FedProx, OnTheConvergence, haddadpour2019convergence, Li2020Fair, Wang2018Cooperative, ClusteredSampling}. By definition, we have $\sum_{i=1}^{n}p_i =1$, and the clients aggregation weights take the form:
\begin{equation}
\omega_i(S_t) = \frac{1}{m} \sum_{k =1 }^m \mathbb{I}(l_k=i) .
\label{eq:MD_aggreg_weight}
\end{equation}

\textbf{Uniform sampling}. 
This scheme samples $m$ clients uniformly without replacement. Since in this case a client is sampled with probability $ p(\{i \in S_t\}) = m/n$, the requirement of Definition \ref{def:sampling} implies:
\begin{equation}
\omega_i(S_t) 
=  \mathbb{I}( i \in S_t ) \frac{n}{m}p_i.
\label{eq:Uniform_aggreg_weight}
\end{equation}

We note that this formulation for Uniform sampling is a generalization of the scheme previously used for FL applications with identical client importance, i.e. $p_i = 1 / n$ \citep{SCAFFOLD, OnTheConvergence, reddi2021adaptive, DynamicFL}. We note that $\VAR{\sum_{i=1}^{n}\omega_i(S_t)} = 0$ if and only if $p_i = 1/n$ for all the clients as, indeed, $\sum_{i=1}^{n}\omega_i(S_t) = m \frac{n}{m} \frac{1}{n} = 1$

With reference to equation (\ref{eq:aggreg_SCAFFOLD_general}), we note that by setting $\eta_g = 1 $, and by imposing the condition $\forall S_t,\ \sum_{i=1}^{n}\omega_i(S_t) = 1$, we retrieve equation (\ref{eq:FedAvg_server_aggregation}). 
This condition is satisfied for example by MD sampling and Uniform sampling for identical clients importance. 

We finally note that the covariance of the aggregation weights for both MD and Uniform sampling satisfies Assumption \ref{ass:covariance}.
\begin{assumption}[Client Sampling Covariance]
    \label{ass:covariance}
    There exists a constant $\alpha$ such that the client sampling covariance satisfies $\forall i\neq j,\  \COV{\omega_i(S_t)}{\omega_j(S_t)} = - \alpha p_i p_j$.
\end{assumption}
We provide in Table \ref{table:sampling} the derivation of $\alpha$ and the resulting covariance for these two schemes with calculus detailed in Appendix \ref{app:sec:sampling_calculus}. Furthermore, this property is common to a variety of sampling schemes, for example based on Binomial or Poisson Binomial distributions (detailed derivations can be found in Appendix \ref{app:sec:sampling_calculus}). Following this consideration, in addition to  Definition \ref{def:sampling}, in the rest of this work we assume the additional requirement for a client sampling scheme to satisfy Assumption \ref{ass:covariance}.

\subsection{Advanced client sampling techniques}

Importance sampling for centralized SGD  \cite{ZhaoImportance,JMLR:Richtarik} has been developed to reduce the variance of the gradient estimator in the centralized setting and provide faster convergence. According to this framework, each data point is sampled according to a probability based on a parameter of its loss function (e.g. its Lipschitz constant), in opposition to classical sampling where clients are sampled with same probability. 
These works cannot be seamlessly applied in FL, since in general no information on the clients loss function should be disclosed to the server. Therefore, the operation of client sampling in FL cannot be seen as an extension of importance sampling.
Regarding advanced FL client sampling, \cite{ClusteredSampling} extended MD sampling to account for collections of sampling distributions with varying client sampling probability. From a theoretical perspective, this approach was proven to have identical convergence guarantees of MD sampling, with albeit experimental improvement justified by  lower variance of the clients' aggregation weights.
In \cite{Richtarik_optimal_sampling}, clients probability are set based on the euclidean norm of the clients local work.
We show in Appendix \ref{app:sec:sampling_calculus} that these advanced client sampling strategies also satisfy our covariance assumption \ref{ass:covariance}, and are thus encompassed by the general theory developed in Section \ref{sec:convergence}.  

\section{Convergence Guarantees}\label{sec:convergence}

Based on the assumptions introduced in Section \ref{sec:related}, in what follows we elaborate a new theory relating the convergence of FL to the statistical properties of client sampling schemes. In particular, Theorem \ref{theo:convergence_paper} quantifies the asymptotic relationship between client sampling and FL convergence.

\subsection{Asymptotic FL convergence with respect to client sampling}\label{sec:convergence_all_steps}

To prove FL convergence with client sampling, our work relies on the following three assumptions 
\citep{FedNova, FedProx, SCAFFOLD, haddadpour2019convergence, Matcha, Wang2019adaptive}:
\begin{assumption}[Smoothness]\label{ass:smoothness}
	The clients local objective function is $L$-Lipschitz smooth, that is, $\forall i \in \{1, ..., n\},\ \norm{\nabla \Lcal_i(x) - \nabla \Lcal_i(y)} \le L \norm{x - y}$.
\end{assumption}


\begin{assumption}[Bounded Dissimilarity ]
	\label{ass:dissimilarity} 
	There exist constants $\beta^2\ge 1$ and $\kappa^2 \ge 0$ such that for every combination of positive weights $\{w_i\}$ such that $\sum_{i=1}^{n}w_i = 1$, we have $\sum_{i=1}^{n} w_i \norm{\nabla \Lcal_i(x)}^2 \le \beta^2 \norm{\nabla \Lcal(x)}^2 + \kappa^2$. If all the local loss functions are identical, then we have $\beta^2 =1$ and $\kappa^2 =0$. 
\end{assumption}

\begin{assumption}[Unbiased Gradient and Bounded Variance]\label{ass:unbiased}
	Every client stochastic gradient $g_i(\vx|B)$ of a model $\vx$ evaluated on batch $B$ is an unbiased estimator of the local gradient. We thus have 
	$ \EE{B}{\vxi_i(B)} = 0$ and $0 \leq\EE{B}{\norm{\vxi_i(B)}^2} \le \sigma^2$, with $\vxi_i(B) = g_i(\vx|B) - \nabla \Lcal_i(\vx)$. 
\end{assumption}

We formalize in the following theorem the relationship between the statistical properties of the client sampling scheme and the asymptotic convergence of FL  (proof in Appendix \ref{app:sec:convergence_Wang}). 

\begin{theorem}[FL convergence]\label{theo:convergence_paper}
Let us consider a client sampling scheme satisfying Definition \ref{def:sampling} and Assumption \ref{ass:covariance}. Under Assumptions  \ref{ass:smoothness}, \ref{ass:dissimilarity}, and \ref{ass:unbiased}, and with sufficiently small local step size $\eta_l$, the following convergence bound holds:
	\begin{align}
	&\frac{1}{T}\sum_{t=0}^{T-1} \E{\norm{\nabla {\Lcal}(\vtheta^t)}^2}
	\le \Ocal \left( \frac{1}{\tilde{\eta} K T} \right)
	\nonumber\\
	&+ \Ocal \left(  \eta_l^2 (K-1) \sigma^2 \right)
	+  \Ocal \left( \tilde{\eta}  \left[ \Sigma + \sum_{i=1}^{n} p_i^2\right] \sigma^2 \right)
	\label{eq:theo_convergence}
	\\
	&+ \Ocal \left( \eta_l^2 K(K-1) \kappa^2 \right)
	+ \Ocal \left( \tilde{\eta} \gamma \left[(K-1)\sigma^2 + K \kappa^2 \right] \right)
	\nonumber
	,
	\end{align}
	where $\tilde{\eta} = \eta_g \eta_l$, $K$ is the number of local SGD,
	\begin{equation}
		\Sigma  
		= \sum_{i=1}^{n} \VAR{\omega_i(S_t)}
    \end{equation}
    and
    \begin{equation}
		\gamma 
		= \sum_{i=1}^{n}\VAR{\omega_i(S_t)} + \alpha \sum_{i=1}^{n} p_i^2 
		.
	\end{equation}
\end{theorem}

We first observe that any client sampling scheme satisfying the assumptions of Theorem \ref{theo:convergence_paper} converges to its optimum. Through $\Sigma$ and $\gamma$, equation (\ref{eq:theo_convergence}) shows that our bound is proportional to the clients aggregation weights through the quantities $\VAR{\omega_i(S_t)}$ and $\alpha$,
which thus should be minimized. These terms are non-negative and are minimized and equal to zero only with full participation of the clients to every optimization round.
Theorem \ref{theo:convergence_paper} does not require the sum of the weights $\omega_i(S_t)$ to be equal to 1. Yet, for client sampling satisfying 
$\VAR{\sum_{i=1}^{n}\omega_i(S_t)} = 0$, we get $\alpha \propto \Sigma$.
Hence, choosing an optimal client sampling scheme amounts at choosing the client sampling with the smallest $\Sigma$. This aspect has been already suggested in \cite{ClusteredSampling}.

The convergence guarantee proposed in Theorem \ref{theo:convergence_paper} extends the work of \cite{FedNova} where, in addition of considering \textsc{FedAvg} with clients performing $K$ vanilla SGD, we include a server learning rate $\eta_g$ and integrate client sampling (equation (\ref{eq:aggreg_SCAFFOLD_general})). With full client participation ($\Sigma = \gamma = 0$) and $\eta_g=1$, we retrieve the convergence guarantees of \cite{FedNova}. 
Furthermore, our theoretical framework can be applied to any client sampling satisfying the conditions of Theorem \ref{theo:convergence_paper}. In turn, Theorem \ref{theo:convergence_paper} holds for full client participation, MD sampling, Uniform sampling, as well as for the other client sampling schemes detailed in Appendix \ref{app:sec:sampling_calculus}. 
Finally, the proof of Theorem \ref{theo:convergence_paper} is general enough to account for FL regularization methods \citep{FedProx, FedDane, FedDyn}, other SGD solvers \citep{Adam, AdaGrad,pmlr-v89-li19c}, and/or gradient compression/quantization \citep{FedPaq, QSparse, Atomo}. For all these applications, the conclusions drawn for client samplings satisfying the assumptions of Theorem \ref{theo:convergence_paper} still hold.

\begin{figure*}
	\includegraphics[width = \textwidth]{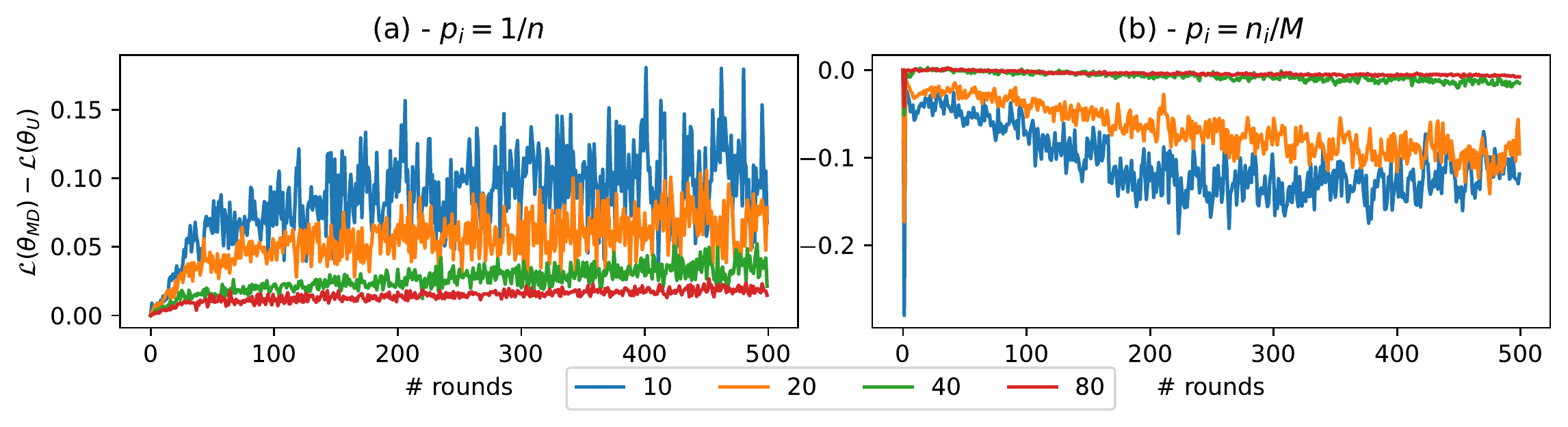}
	\caption{Difference between the convergence of the global losses resulting from MD and Uniform sampling when considering $n \in \{10, 20, 40, 80 \}$ clients and sampling $m = n / 2$ of them. In (a), clients have identical importance, i.e. $p_i = 1/n$. In (b), clients importance is proportional to their amount of data, i.e. $p_i = n_i / M$. Differences in global losses are averaged across 30 FL experiments with different model initialization (global losses are provided in Figure \ref{fig:shakespeare_evo_n}). }
	\label{fig:shakespeare}
\end{figure*}

\subsection{Application to current client sampling schemes}

\textbf{MD sampling}. When using Table \ref{table:sampling} to compute $\Sigma$ and $\gamma$ close-form we obtain:
\begin{equation}
\Sigma_{MD}
= \frac{1}{m} \left[1- \sum_{i=1}^{n}p_i^2\right]
\text{ and }
\gamma_{MD}
= \frac{1}{m}
,
\end{equation}
where we notice that $\Sigma_{MD} \le \frac{1}{m} = \gamma_{MD}$. Therefore, one can obtain looser convergence guarantees than the ones of Theorem \ref{theo:convergence_paper}, independently from the amount of participating clients $n$ and set of clients importance $\{p_i\}$, while being inversely proportional to the amount of sampled clients $m$. The resulting bound shows that FL with MD sampling converges to its optimum for any FL application.

\textbf{Uniform sampling}. Contrarily to MD sampling, the stochastic aggregation weights of Uniform sampling do not sum to 1. As a result, we can provide FL scenarios diverging when coupled with Uniform sampling. Indeed, using Table \ref{table:sampling} to compute $\Sigma$ and $\gamma$ close-form we obtain
\begin{equation}
\Sigma_U
= \left[\frac{n}{m} - 1\right]\sum_{i=1}^{n}p_i^2, 
\end{equation}
and
\begin{equation}
\gamma_{U} 
= \left[1 + \frac{1}{n-1}\right]\left[\frac{n}{m} - 1 \right] \sum_{i=1}^{n}p_i^2
,
\end{equation}
where we notice that $\gamma_{U} = \left[1 + \frac{1}{n-1}\right] \Sigma_{U}$. Considering that $\sum_{i=1}^{n}p_i^2 \le 1$, we have $\Sigma_{U} \le \frac{n}{m} - 1$, which goes to infinity for large cohorts of clients and thus prevents FL with Uniform sampling to converge to its optimum. 
Indeed, the condition $\sum_{i=1}^{n}p_i^2 \le 1$ accounts for every possible scenario of client importance  $\{p_i\}$, including the very heterogeneous ones. In the special case where $p_i = 1/n$, we have $\sum_{i=1}^{n}p_i^2 = 1/n$, such that $\Sigma_{U}$ is inversely proportional to both $n$ and $m$. Such FL applications converge to the optimum of equation (\ref{eq:global_loss}) for any configuration of $n$, $\{p_i\}$ and $m$.

Moreover, the comparison between the quantities $\Sigma$ and $\gamma$ for MD and Uniform sampling shows that Uniform sampling outperforms 
MD sampling when $p_i = 1/n$. More generally, Corollary \ref{cor:sufficient_conditions} provides sufficient conditions with Theorem \ref{theo:convergence_paper} for Uniform sampling to have better convergence guarantees  than MD sampling (proof in Appendix \ref{app:sec:sufficient_conditions}). 
\begin{corollary}\label{cor:sufficient_conditions}
	Uniform sampling has better convergence guarantees than MD sampling when $\Sigma_U \le \Sigma_{MD}$, and 
	$\gamma_{U} \le \gamma_{MD}$ which is equivalent to
	\begin{equation}
	\sum_{i=1}^{n}p_i^2 \le \frac{1}{n-m +1}
	.
	\end{equation}
\end{corollary}
Corollary \ref{cor:sufficient_conditions} can be related to $\VAR{\sum_{i=1}^{n}\omega_i(S_t)}$, the variance for the sum of the aggregation weights, which is always null for MD sampling, and different of 0 for Uniform sampling except when $p_i = 1/n$ for all the clients. 

A last point of interest for the comparison between MD and Uniform sampling concerns the respective time complexity for selecting clients. Sampling with a Multinomial Distribution has time complexity $\Ocal(n + m \log(n))$, where $\Ocal(n)$ comes from building the probability density function to sample clients indices \citep{categorical_distribution}. This makes
MD sampling difficult to compute or even intractable for large cohorts of clients. 
On the contrary sampling $m$ elements without replacement from $n$ states is a reservoir sampling problem and takes time complexity $\Ocal(m(1 + \log(n/m))$\citep{ReservoirSampling}.
In practice, clients either receive identical importance ($p_i = 1/n$) or an importance proportional to their data ratio, for which we may assume computation $p_i = \Ocal(1/n)$. As a result, for important amount $n$ of participating clients, Uniform sampling should be used as the default client sampling due to its lower time complexity. However, for small amount of clients and heterogeneous client importance, MD sampling should be used by default.

Due to space constraints, we only consider in this manuscript applying Theorem \ref{theo:convergence_paper} to Uniform and MD sampling, which can also be applied to Binomial and Poisson Binomial sampling introduced in Section \ref{app:sec:sampling_calculus}, and satisfying our covariance assumption. 
To the best of our knowledge, we could only find \textit{Clustered sampling} introduced in \cite{ClusteredSampling} not satisfying this assumption. Still, with minor changes, we provide for this sampling scheme a similar bound to the one of Theorem \ref{theo:convergence_paper} (Appendix \ref{app:sec:clustered_convergence}), ultimately proving that clustered sampling improves MD sampling.

\section{Experiments on real data}\label{sec:experiments}

\begin{figure*}
	\includegraphics[width = \textwidth]{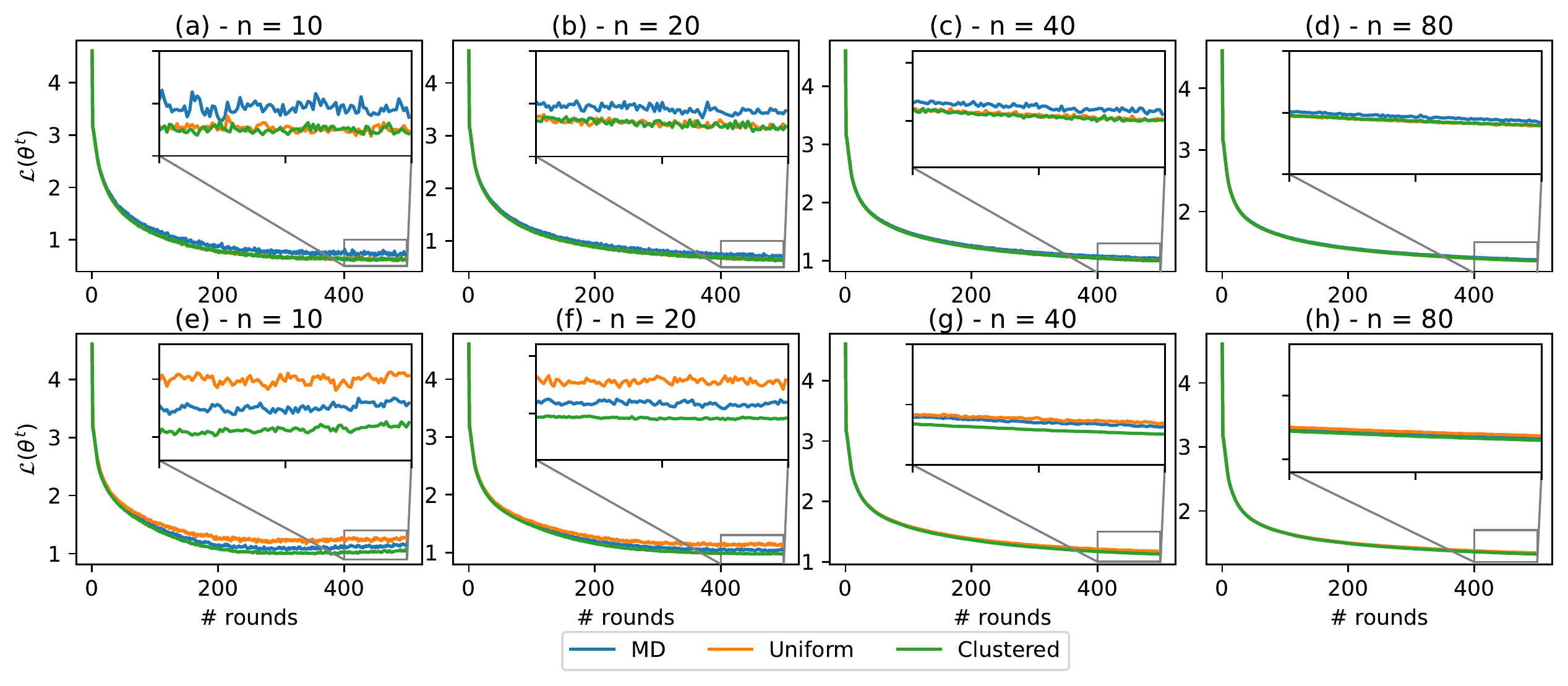}
	\caption{Convergence of the global losses for MD, Uniform, and Clustered sampling when considering $n \in \{10, 20, 40, 80 \}$ clients and sampling $m = n / 2$ of them. In (a-d), clients have identical importance, i.e. $p_i = 1/n$. In (e-h), clients importance is proportional to their amount of data, i.e. $p_i = n_i / M$. 
	Zoom of the global losses over the last 100 server aggregations and a variation of 0.5 in the global loss.
	}
	\label{fig:shakespeare_evo_n}
\end{figure*}

In this section, we provide an experimental demonstration of the convergence properties identified in Theorem \ref{theo:convergence_paper}. \footnote{Code and data are available at \url{https://github.com/Accenture/Labs-Federated-Learning/tree/impact_client_sampling}.}
We study a LSTM model for next character prediction on the dataset of \textit{The complete Works of William Shakespeare } \citep{FedAvg, Leaf}. We use a two-layer LSTM classifier containing 100 hidden units with an 8 dimensional embedding layer. The model takes as an input a sequence of 80 characters, embeds each of the characters into a learned 8-dimensional space and outputs one character per training sample after 2 LSTM layers and a fully connected one. 

When selected, a client performs $K=50$ SGD steps on batches of size $B=64$ with local learning rate $\eta_l = 1.5$. The server considers the clients local work with $\eta_g =1$.  We consider $n \in \{10,20,40,80\}$ clients, and sample half of them at each FL optimization step. 
While for sake of interpretability we do not apply a decay to local and global learning rates, we  note that our theory remains unchanged even in presence  of a learning rate decay.  In practice, for dataset with important heterogeneity, considering $\eta_g < 1$ can speed-up FL with a more stable convergence.

We compare the impact of MD, Uniform, and Clustered sampling, on the convergence speed of \textsc{FedAvg}.  
With Clustered sampling, the server selects $m$ clients from $m$ different clusters of clients created based on the clients importance \citep[Algorithm 1]{ClusteredSampling}. MD sampling is a special case of Clustered sampling, where every cluster is identical.

\textbf{Clients have identical importance }[$p_i = 1 /n$]. 
We note that Uniform sampling consistently outperforms MD sampling due to the lower covariance parameter, while the improvement between the resulting convergence speed is inversely proportional to the number of participating clients $n$ (Figure \ref{fig:shakespeare}a and Figure \ref{fig:shakespeare_evo_n}a-d). This result confirms the derivations of Section \ref{sec:convergence}.
Also, with Clustered sampling and identical client importance, every client only belongs to one cluster. 
Hence, Clustered sampling reduces to Uniform sampling and we retrieve identical convergence for both samplings (Figure \ref{fig:shakespeare_evo_n}a-d). This point was not raised in \cite{ClusteredSampling}.

\textbf{Clients importance depends on the respective data ratio } [$p_i = n_i / M$]. 
In this experimental scenario the aggregation weights for Uniform sampling do not always sum to 1, thus leading to the slow-down of FL convergence. Hence, we see in Figure \ref{fig:shakespeare}b that MD always outperforms Uniform sampling. This experiment shows that the impact on FL convergence of the variance of the sum of the stochastic aggregation weights is more relevant than the one due to the covariance parameter $\alpha$.  
We also retrieve in Figure \ref{fig:shakespeare_evo_n}e-h that Clustered sampling always outperform MD sampling, which confirms that for two client samplings with a null variance of the sum of the stochastic aggregation weights, the one with the lowest covariance parameter $\alpha$ converges faster. 
We also note that the slow-down induced by the variance is reduced when more clients do participate.
This is explained by the fact that the standard deviation of the clients data ratio is reduced with larger clients participation, e.g. $p_i = 1/10 \pm 0.13$ for $n=10$ and $p_i = 1/80 \pm 0.017$ for $n=80$.
We thus conclude that the difference between the effects of MD, Uniform, and Clustered sampling is mitigated with a large number of participating clients (Figure \ref{fig:shakespeare}b and Figure\ref{fig:shakespeare_evo_n}e-h). 

Additional experiments on Shakespeare are provided in Appendix \ref{app:sec:additional_experiments}. We show the influence of the amount of sampled clients $m$
and amount of local work $K$ 
on the convergence speed of MD and Uniform sampling.

Finally, additional experiments on CIFAR10 \citep{CIFAR-10} are provided in Appendix \ref{app:sec:additional_experiments}, where we replicate the experimental scenario previously proposed in \cite{ClusteredSampling}. In these applications, 100 clients are partitioned using a Dirichlet distribution which provides federated scenarios with different level of heterogeneity. For all the experimental scenarios considered, both results and conclusions are in agreement with those here derived for the Shakespeare dataset.

\section{Conclusion}

In this work, we highlight the asymptotic impact of client sampling on FL with Theorem \ref{theo:convergence_paper}, and shows that the convergence speed is inversely proportional to both the sum of the variance of the stochastic aggregation weights, and to their covariance parameter $\alpha$. Moreover, to the best of our knowledge, this work is the first one accounting for schemes where the sum of the weights is different from 1. 

Thanks to our theory, we investigated MD and Uniform sampling from both theoretical and experimental standpoints. We established that when clients have approximately identical importance, i.e $p_i = 1/n$, Uniform outperforms MD sampling, due to the larger impact of the covariance term for the latter scheme. On the contrary, Uniform sampling is outperformed by MD sampling in more general cases, due to the slowdown induced by its stochastic aggregation weights not always summing to 1. Yet, in practical scenario with very large number of clients, MD sampling may be unpractical, and Uniform sampling could be preferred due to the more advantageous time complexity. 

In this work, we also showed that our theory encompasses advanced FL sampling schemes, such as the one recently proposed in \cite{ClusteredSampling}, and \cite{Richtarik_optimal_sampling}.
Finally, while the contribution of this work is in the study of the impact of a client sampling on the global optimization objective, further extensions may focus on the analysis of the impact of clients selection method on individual users’ performance, especially in presence of heterogeneity.

\section* {Acknowledgements}

This work has been supported by the French government, through the 3IA Côte d’Azur Investments in the Future project managed by the National Research Agency (ANR) with the reference number ANR-19-P3IA-0002, and by the ANR JCJC project Fed-BioMed 19-CE45-0006-01. The project was also supported by Accenture.
The authors are grateful to the OPAL infrastructure from Université Côte d'Azur for providing resources and support.

\bibliographystyle{named}
\bibliography{biblio}

\newpage
\appendix
\onecolumn
\newpage

\section{Client Sampling Schemes Calculus}\label{app:sec:sampling_calculus}

In this section, we calculate for MD, Uniform, Poisson, and Binomial sampling the respective aggregation weight variance $\VAR{\omega_i(S_t)}$, the covariance parameter $\alpha$  such that $\COV{\omega_i(S_t))}{\omega_j(S_t)} = - \alpha p_i p_j$, and the variance of the sum of weights $\VAR{\sum_{i=1}^{n}\omega_i(S_t)}$.  We also propose statistics for the parameter $N$,  i.e. the amount of clients the server communicates with at an iteration:
\begin{equation}
N = \sum_{i=1}^{n}  \mathbb{I}( i\in S_t ).
\end{equation}

\subsection{Property \ref{prop:sum_weights}}\label{app:sec:properties}

\begin{proposition}\label{prop:sum_weights}
	For any client sampling, we have $0 \le \alpha \le 1$ and 
	\begin{equation}
		\VAR{\sum_{i=1}^n \omega_i(S_t)} 
		= \sum_{i=1}^n \VAR{\omega_i(S_t)}  - \alpha \left[ 1 - \sum_{i=1}^n   p_i^2\right]
		.
	\end{equation}
\end{proposition}

\begin{proof}
	\textbf{Covariance parameter}
	\begin{equation}
	\COV{\omega_i(S_t)}{\omega_j(S_t)}
	= \E{\omega_i(S_t)\omega_j(S_t)} - p_i p_j
	\ge - p_i p_j
	.
	\end{equation}
    Hence, we have $\alpha \le 1.$
    
	\textbf{Aggregation Weights Sum}
	\begin{align}
	\VAR{\sum_{i=1}^{n}\omega_i(S_t)}
	&= \sum_{i=1}^{n}\VAR{\omega_i(S_t)}
	+ \sum_{i, j\neq i} \COV{\omega_i(S_t)}{\omega_j(S_t)}\\
	&= \sum_{i=1}^{n}\VAR{\omega_i(S_t)}
	- \alpha \sum_{i, j\neq i} p_i p_j\\
	&= \sum_{i=1}^{n}\left[ \VAR{\omega_i(S_t)}- \alpha p_i (1 - p_i)\right] \label{app:eq:V_1}\\
	&= \sum_{i=1}^n \VAR{\omega_i(S_t)}  - \alpha \left[ 1 - \sum_{i=1}^n   p_i^2\right]
	,
	\end{align}
	where we use $\sum_{i=1}^{n}p_i = 1$, equation (\ref{eq:global_loss}), for the third and fourth equality. 
	
	\textbf{Re-expressing $\alpha$}. Using equation (\ref{app:eq:V_1}), we get
	\begin{align}
	\VAR{\sum_{i=1}^{n}\omega_i(S_t)}
	&= \sum_{i=1}^{n} \VAR{\omega_i(S_t)} - \alpha \left[1 - \sum_{i=1}^{n} p_i^2\right],
	\end{align}
	which, with reordering, gives
	\begin{equation}
	\alpha
	= \frac{\sum_{i=1}^{n}\VAR{\omega_i(S_t)} - \VAR{ \sum_{i=1}^{n}  \omega_i(S_t)}}{1 - \sum_{i=1}^{n} p_i^2} .
	\end{equation}
	
\end{proof}

\subsection{No sampling scheme}
When every client participate at an optimization round, we have $\omega_i(S_t) = p_i$ which gives $\VARR{S_t}{\omega_i(S_t)} =0$, $\alpha = 0$, and $N = n$.

\subsection{MD sampling}\label{app:sec:MD}

We recall equation (\ref{eq:MD_aggreg_weight}),
\begin{equation}
\omega_i(S_t) = \frac{1}{m} \sum_{k =1 }^m  \mathbb{I}(l_k=i) 
,
\end{equation}
which gives
\begin{align}
\E{\omega_i(S_t)\omega_j(S_t)} 
&= \frac{1}{m^2} \sum_{k, l\neq k}\E{  \mathbb{I}(l_k=i)   \mathbb{I}(l_l=j) } 
+ \frac{1}{m^2} \sum_{k=1}^m\E{  \mathbb{I}(l_k=i)   \mathbb{I}(l_k=j) }\\
&= \frac{1}{m^2} \sum_{k, l\neq k}p_i p_j
+ \frac{1}{m^2} \sum_{k=1}^m\E{  \mathbb{I}(l_k=i)   \mathbb{I}(l_k=j)  }\\
&= \frac{m-1}{m} p_i p_j
+ \frac{1}{m} \E{  \mathbb{I}(l=i)   \mathbb{I}(l=j) }
%
\label{app:eq:R0}
\end{align}

\textbf{Variance}($i=j$). We get $\E{ \mathbb{I}(l=i)   \mathbb{I}(l=j) } = \E{ \mathbb{I}(l=i)   } = p_i$, which gives:
\begin{equation}
	\VAR{\omega_i(S_t)}
	= -\frac{1}{m} p_i^2
	+ \frac{1}{m} p_i
	\label{app:eq:R1}
\end{equation}

\textbf{Covariance}($i\neq j$). We get $\E{\mathbb{I}(l=i)   \mathbb{I}(l=j)} = 0$, which gives:
\begin{equation}
\COV{\omega_i(S_t)}{\omega_j(S_t)}
= -\frac{1}{m} p_i p_j
,
\end{equation}
and by definition we get
\begin{equation}
	\alpha = \frac{1}{m}
	\label{app:eq:R2}
\end{equation}

\textbf{Aggregation Weights Sum}.
Using equation (\ref{app:eq:R1})) and (\ref{app:eq:R2}) 
with Property \ref{prop:sum_weights}
, we get
\begin{equation}
	\VAR{\sum_{i=1}^{n}\omega_i(S_t)} 
	= 0.
	\label{app:eq:R3}
\end{equation}

\textbf{Amount of clients}. Considering that $p(i \in S_t) = 1 - p(i \notin S_t) = 1 - (1- p_i)^m$, we get:
\begin{equation}
\E{N} = \sum_{i=1}^{n}\mathbb{P}(i\in S_t) 
= n - \sum_{i=1}^{n}(1 - p_i)^m \le m
\end{equation}


\subsection{Uniform Sampling}

We recall equation (\ref{eq:Uniform_aggreg_weight}),
\begin{equation}
\omega_i(S_t) 
=  \mathbb{I}( i \in S_t ) \frac{n}{m}p_i.
\end{equation}

\textbf{Variance}. 
We first calculate the probability for a client to be sampled, i.e. 
\begin{align}
	\mathbb{P}(i \in S_t) = 1 - \mathbb{P}(i \notin S_t) = 1 - \frac{n-1}{n}... \frac{n-m}{n-m+1} = 1 - \frac{n-m}{n} = \frac{m}{n}.
	\label{app:eq:S1}
\end{align}

Using equation (\ref{app:eq:S1}), we have
\begin{equation}
\VARR{S_t}{\omega_i(S_t)} 
=\left[\frac{n}{m}p_i\right]^2\VAR{ \mathbb{I}( i \in S_t ) }
=\frac{n^2}{m^2}\frac{m}{n} (1- \frac{m}{n})p_i^2
=(\frac{n}{m}- 1)p_i^2
\label{app:eq:S2}
\end{equation}

\textbf{Covariance}. 
We have
\begin{align}
\mathbb{P}(\{i, j\} \in S_t) 
&= \mathbb{P}(i \in S_t) 
+ \mathbb{P}(j \in S_t) 
- \mathbb{P}(i \cup j \in S_t)\\
&= \mathbb{P}(i \in S_t) 
+ \mathbb{P}(j \in S_t) 
-  (1 - \mathbb{P}(\{i, j\} \notin S_t) ),
\label{app:eq:S3}
\end{align}
and 
\begin{equation}
	\mathbb{P}(\{i, j\} \notin S_t) 
	= \frac{n-2}{n}... \frac{n - m -1}{n - m +1} 
	= \frac{(n-m)(n-m -1)}{n(n-1)}.
\label{app:eq:S4}
\end{equation}

Substituting equation (\ref{app:eq:S1}) and (\ref{app:eq:S4}) in equation (\ref{app:eq:S3}) gives
\begin{align}
\mathbb{P}(\{i, j\} \in S_t) 
&= 2 \frac{m}{n}
-  1 
+ \frac{(n-m)(n-m -1)}{n(n-1)}\\
&= \frac{1}{n(n-1)} \left[2 m (n - 1) - n(n-1) + (n-m)(n-m -1)\right]\\
&= \frac{m(m-1)}{n(n-1)} 
.
\end{align}

Hence, we can express the aggregation weights covariance as
\begin{equation}
\COV{\omega_i(S_t)}{\omega_j(S_t)} 
= \frac{n^2}{m^2}\frac{m (m-1)}{n(n-1)}p_j p_k - p_j p_k
,
\end{equation}
which gives
\begin{equation}
\alpha = \frac{n - m}{m(n-1)}.
\label{app:eq:S5}
\end{equation}

\textbf{Aggregation Weights Sum}.
Combining equation (\ref{app:eq:S2}) and (\ref{app:eq:S5}) with Property \ref{prop:sum_weights} gives
\begin{align}
\VAR{\sum_{i=1}^{n}\omega_i(S_t)} 
= \sum_{i=1}^{n} \left[\frac{n}{m} - 1\right] p_i^2 
- \frac{n-m}{m(n-1)}\sum_{i=1}^{n}p_i(1-p_i)
= \frac{n-m}{m(n-1)}\left[n\sum_{i=1}^{n}p_i^2 -1\right],
\label{app:eq:S6}
\end{align}
where we retrieve $\VAR{\sum_{i=1}^{n}\omega_i(S_t)} = 0$ for identical client importance, i.e. $\sum_{i=1}^{n}p_i^2 = \frac{1}{n}$.

\textbf{Amount of Clients}. $N = m$.

\subsection{Poisson Binomial Distribution}

Clients are sampled according to a Bernoulli with a probability proportional to their importance $p_i$, i.e.
\begin{equation}
\omega_i(S_t) 
= \frac{1}{m}\mathbb{B}(mp_i).
\end{equation}
Hence, only $m \ge p_{max}^{-1}$ can be sampled and we retrieve $\E{\omega_i(S_t)} = \frac{1}{m} m p_i = p_i$. 

\textbf{Variance}. 
\begin{equation}
\VARR{S_t}{\omega_i(S_t)} 
=\frac{1}{m^2} m p_i (1- m p_i)
=\frac{1}{m} p_i (1- m p_i)
\end{equation}

\textbf{Covariance}. Due to the independence of each stochastic weight, we also get:
\begin{equation}
\COV{\omega_i(S_t)}{\omega_j(S_t)} = 0 
\end{equation}
\textbf{Aggregation Weights Sum}.
Using Property \ref{prop:sum_weights} we obtain
\begin{equation}
\VAR{\sum_{i=1}^{n}\omega_i(S_t)} 
= \frac{1}{m} - \sum_{i=1}^{n} p_i^2 
.
\end{equation}

\textbf{Amount of Clients}.
\begin{equation}
\E{N} = m 
\text{ and }
\VAR{N} = m - m^2 \sum_{i=1}^{n}p_i^2. 
\end{equation}

\subsection{Binomial Distribution}

Clients are sampled according to a Bernoulli with identical sampling probability, i.e.
\begin{equation}
\omega_i(S_t) 
= \frac{n}{m}\mathbb{B}(\frac{m}{n})p_i.
\end{equation}
Hence, we retrieve $\E{\omega_i(S_t)} = \frac{n}{m} \frac{m}{n} p_i = p_i$. 

\textbf{Variance}.
\begin{equation}
\VARR{S_t}{\omega_i(S_t)} 
=\frac{n^2}{m^2} \frac{m}{n}  (1- \frac{m}{n}) p_i^2
=\frac{n-m}{m} p_i^2 .
\end{equation}

\textbf{Covariance}.
Due to the independence of each stochastic weight, we have:
\begin{equation}
\COV{\omega_i(S_t)}{\omega_j(S_t)} = 0. 
\end{equation}

\textbf{Aggregation Weights Sum}.
Using Property \ref{prop:sum_weights} gives
\begin{equation}
\VAR{\sum_{i=1}^{n}\omega_i(S_t)} 
= \frac{n-m}{m}\sum_{i=1}^{n} p_i^2 
.
\end{equation}

\textbf{Amount of Clients}.
\begin{equation}
\E{N} = m 
\text{ and }
\VAR{N}
= m - \frac{m^2}{n}.
\end{equation}

%
%
%

\subsection{Clustered Sampling}\label{app:sec:clustered}

Clustered sampling \citep{ClusteredSampling} is a generalization of MD sampling where instead of sampling $m$ clients from the same distributions, $m$ clients are sampled from $m$ different distributions $\{W_k\}_{k=1}^m$ each of them privileging a different subset of clients. We denote by $r_{k, i}$ the probability of client $i$ to be sampled in distribution $k$. To satisfy Definition \ref{def:sampling}, the original work \citep{ClusteredSampling} provides the conditions:
\begin{equation}
	\forall k \in \{1, ..., m\},\ \sum_{i=1}^{n}r_{k, i} = 1
	\text{ and }
	\forall i \in \{1, ..., n\},\ \sum_{k=1}^{m}r_{k, i} = m p_i
	.
	\label{app:eq:clustered_prop}
\end{equation}
The clients aggregation weights remain identical to the one of MD sampling, i.e.
\begin{equation}
	\omega_i(S_{Cl}) = \frac{1}{m}\sum_{k=1}^{K}\mathbb{I}(l_k = i)
	,
\end{equation}
where $\mathbb{I}(l_k = i)$ are still independently distributed but not identically.

We have
\begin{align}
\E{\omega_i(S_t)\omega_j(S_t)} 
&= \frac{1}{m^2} \sum_{k, l\neq k}\E{  \mathbb{I}(l_k=i)   \mathbb{I}(l_l=j) } 
+ \frac{1}{m^2} \sum_{k=1}^m\E{  \mathbb{I}(l_k=i)   \mathbb{I}(l_k=j) }\\
&= \frac{1}{m^2} \sum_{k, l\neq k} r_{k, i} r_{l, j}
+ \frac{1}{m^2} \sum_{k=1}^m\E{  \mathbb{I}(l_k=i)   \mathbb{I}(l_k=j)  }\\
&= p_i p_j 
-\frac{1}{m^2} \sum_{k=1}^{m} r_{k, i}r_{k, j}
+ \frac{1}{m^2} \sum_{k=1}^m\E{  \mathbb{I}(l_k=i)   \mathbb{I}(l_k=j)  }
%
,
\end{align}
where we retrieve equation (\ref{app:eq:R0}) when $r_{k, i} = p_i$.

\textbf{Variance ($ i = j$)}. We get $\E{ \mathbb{I}(l_k=i)   \mathbb{I}(l_k=j) } = \E{ \mathbb{I}(l_k=i)   } = r_{k, i}$, which gives:
\begin{equation}
	\VAR{\omega_i(S_{Cl})}
	= \frac{1}{m} p_i
	- \frac{1}{m^2}\sum_{k=1}^{m}r_{k, i}^2
	\le \VAR{\omega_i(S_{MD})},
	\label{app:eq:T1}
\end{equation}
where the inequality comes from using the Cauchy-Schwartz inequality with equality if and only if all the $m$ distributions are identical, i.e. $r_{k, i} = p_i$. 

\textbf{Covariance ($ i \neq j$)}. We get $\E{\mathbb{I}(l_k=i)   \mathbb{I}(l_k=j)} = 0$, which gives:
\begin{equation}
	\COV{\omega_i(S_{Cl})}{\omega_j(S_{Cl})}
	= - \frac{1}{m^2}\sum_{k=1}^{m}r_{k, i}r_{k, j}
	\le \COV{\omega_i(S_{MD})}{\omega_j(S_{MD})},
	\label{app:eq:T2}
\end{equation}
where the inequality comes from using the Cauchy-Schwartz inequality with equality if and only if all the $m$ distributions are identical, i.e. $r_{k, i} = p_i$. 

\textbf{Aggregation Weights Sum}
\begin{equation}
	\VAR{\sum_{i=1}^{n} \omega_i(S_{Cl})} 
	= 0.
	\label{app:eq:T3}
\end{equation}

\subsection{Optimal Sampling}
{With optimal sampling \citep{Richtarik_optimal_sampling}, clients are sampled according to a Bernoulli distribution with probability $q_i$, i.e.
\begin{equation}
\omega_i(S_t) 
= \frac{p_i}{q_i}\mathbb{B}(q_i).
\end{equation}
Hence, we retrieve $\E{\omega_i(S_t)} 
= \frac{p_i}{q_i} q_i
= p_i$. 

\textbf{Variance}.
\begin{equation}
\VARR{S_t}{\omega_i(S_t)} 
= \frac{1- q_i}{q_i} p_i^2
.
\end{equation}

\textbf{Covariance}.
Due to the independence of each stochastic weight, we have:
\begin{equation}
\COV{\omega_i(S_t)}{\omega_j(S_t)} = 0. 
\end{equation}

\textbf{Aggregation Weights Sum}.
Using Property \ref{prop:sum_weights} gives
\begin{equation}
\VAR{\sum_{i=1}^{n}\omega_i(S_t)} 
= \sum_{i=1}^{n} \frac{1- q_i}{q_i} p_i^2 
.
\end{equation}

\textbf{Amount of Clients}.
\begin{equation}
\E{N} = \sum_{i=1}^n q_i
\text{ and }
\VAR{N}
= \sum_{i=1}^n q_i ( 1- q_i).
\end{equation}}

\begin{table}
	\caption{Common Notation Summary.}
	\label{table:parameters}
	\centering
	\begin{tabular}{cc}
		\hline
		Symbol     & Description    \\
		\hline
		$n$ & Number of clients. \\
		$K$ & Number of local SGD.  \\
		$\eta_l$ & Local/Client learning rate.\\
		$\eta_g$ & Global/Server learning rate.\\
		$\tilde{\eta}$ & Effective learning rate, $\tilde{\eta} = \eta_l \eta_g$.\\
		
		$\vtheta^t$ & Global model at server iteration $t$.  \\
		$\vtheta^*$ & Optimum of the federated loss function, equation (\ref{eq:global_loss}). \\
		$\vtheta_i^{t+1}$ & Local update of client $i$ on model $\theta^t$.\\
		$\vy_{i, k}^{t}$ & Local model of client i after $k$ SGD ($y_{i, K}^{t} = \theta_i^{t+1}$ and $y_{i, 0}^{t} = \theta^t$). \\
		
		$p_i$ & Importance of client $i$ in the federated loss function, equation (\ref{eq:global_loss}).  \\
		$m$ & Number of sampled clients .\\
		$S_t$ & Set of participating clients considered at iteration $t$.\\
		$\omega_i(S_t)$ & Aggregation weight for client $i$ given $S_t$.\\
		$\alpha$ & Covariance parameter.\\
		$\gamma_i$ & cf Section \ref{sec:convergence}\\
		
		$\EE{t}{\cdot}$ & Expected value conditioned on $\theta^t$.\\
		$\Lcal(\cdot)$ & Federated loss function, equation \ref{eq:global_loss}\\
		$\Lcal_i(\cdot)$ & Local loss function of client $i$.\\
		$g_i(\cdot)$ & SGD. We have $\EE{\xi_i}{g_i(\cdot)} = \nabla \Lcal_i(\cdot)$ with Assumption \ref{ass:unbiased}.\\
		$\xi_i$ & Random batch of samples from client $i$ of size $B$.\\

		$L$ & Lipschitz smoothness parameter, Assumption \ref{ass:smoothness}.\\
		$\sigma^2$ & Bound on the variance of the stochastic gradients, Assumption \ref{ass:unbiased}.\\
		$\beta$, $\kappa$ & Assumption \ref{ass:dissimilarity} parameters on the clients gradient bounded dissimilarity.\\

		\hline
	\end{tabular}
\end{table}

\begin{algorithm}[tb]
	\caption{Federated Learning based on equation (\ref{eq:aggreg_SCAFFOLD_general})}\label{alg:pseudo_code}
	\begin{algorithmic}
		\STATE The server sends to the $n$ clients the learning parameters ($K$, $\eta_l$, $B$).
		
		\FOR{$t=0$ {\bfseries to} $T - 1$}
		    \STATE Sample a set of clients $S_t$ and get their aggregation weights $d_i(t)$.
		    
		    \STATE Send to clients in $S_t$ the current global model $\vtheta^t$.
		    
		    \STATE Receive each sampled client contributions $c_i(t) = \vtheta_i^{t+1} -\vtheta^t$.
		    
		    \STATE Creates the new global model $\vtheta^{t+1} = \vtheta^t + \eta_g \sum_{i=1}^n d_i(t) c_i(t) $.
		
		\ENDFOR
	\end{algorithmic}
\end{algorithm}

\section{FL Convergence}\label{app:sec:convergence_Wang}

In Table \ref{table:parameters}, we provide the definition of the different notations used in this work. 
We also propose in Algorithm \ref{alg:pseudo_code} the pseudo-code for \textsc{FedAvg} with aggregation scheme (\ref{eq:aggreg_SCAFFOLD_general}).
Our work is based on the one of \cite{FedNova}. We use the developed theoretical framework they proposed to prove Theorem \ref{theo:convergence_paper}. The focus of our work (and Theorem \ref{theo:convergence_paper}) is on \textsc{FedAvg}. Yet, the proof developed in this section, similarly to the one of \cite{FedNova}, expresses $a_i$ in such a way they can account for a wide-range of regularization method on \textsc{FedAvg}, or optimizers different from Vanilla SGD. This proof can easily be extended to account for different amount of local work from the clients \citep{FedNova}. 

Before developing the proof of Theorem \ref{theo:convergence_paper} in Section \ref{app:sec:proof_convergence}, we introduce the notation we use in Section \ref{app:sec:notation}, some useful lemmas in Section \ref{app:sec:lemmas} and Theorem \ref{theo:intermediary} generalizing Theorem \ref{theo:convergence_paper} in Section \ref{app:sec:intermediary_theo}.

\subsection{Notations}\label{app:sec:notation}
We define by $\vy_{i, k}^t$ the local model of client $i$ after $k$ SGD steps initialized on $\vtheta^t$, which enables us to also define the normalized stochastic gradients $\vd_i^t$ and the normalized gradient $\vh_i^t$ defined as
\begin{equation}
\vd_i^t = \frac{1}{a_i} \sum_{k=0}^{K-1} a_{i,k} g_i(\vy_{i, k}^t)
\text{ and }
\vh_i^t = \frac{1}{a_i} \sum_{k=0}^{K-1} a_{i, k} \nabla \Lcal_i (\vy_{i, k}^t),
\label{app:eq:formalism_sum_grad}
\end{equation}
where $a_{i, k}$ is an arbitrary scalar applied by the client to its $k$th gradient, $\va_i = [a_{i, 0}, .., a_{i, K -1}]^T$, and $a_i = \norm{\va_i}_1$. In the special case of \textsc{FedAvg}, we have $\va_i = [1, ..., 1]$ and in the one of \textsc{FedProx}, we have $\va_i = [(1 - \mu)^{K-1}, ..., 1]$ where $\mu$ is the \textsc{FedProx} regularization parameter. 

With the formalism of equation (\ref{app:eq:formalism_sum_grad}), we can express a client contribution as $\vtheta_i^{t+1} - \vtheta^t  = - \eta_l a_i \vd_i^t $ and rewrite the server aggregation scheme defined in equation (\ref{eq:aggreg_SCAFFOLD_general}) as 
\begin{align}
\vtheta^{t+1} - \vtheta^t  
&= - \eta_g \eta_l \sum_{i=1}^{n}\omega_i a_i \vd_i^t
,
\end{align}
which in expectation over the set of sampled clients $S_t$ gives
\begin{align}
\EE{S_t}{\vtheta^{t+1} - \vtheta^t}  
= - \tilde{\eta} \sum_{i=1}^{n}p_i a_i \vd_i^t
= - \tilde{\eta} \underbrace{\left(\sum_{i=1}^{n} p_i a_i \right)}_{K_{eff}} \sum_{i=1}^{n} \underbrace{\left(\frac{p_i a_i}{\sum_{i=1}^{n} p_i a_i}\right)}_{w_i} \vd_i^t.
\label{app:eq:expected_client_contrib}
\end{align}

We define the surrogate objective $\tilde{\Lcal}(\vx) = \sum_{i=1}^{n} w_i \Lcal_i(\vx)$, where $\sum_{i=1}^{n}w_i =1 $. 

In what follows, the norm used for $\va_i$ can either be L1, $\norm{\cdot}_1$, or L2, $\norm{\cdot}_2$, For other variables, the norm is always the euclidean  one and $\norm{\cdot}$ is used instead of $\norm{\cdot}_2$. Also, regarding the client sampling metrics, for ease of writing, we use $\omega_i$ instead of $\omega_i(S_t)$ due to the independence of the client sampling statistics with respect to the current optimization round.

\subsection{Useful Lemmas}\label{app:sec:lemmas}

\begin{lemma}\label{lem:decompo_Xi}
	Let us consider $n$ vectors $\vx_i, ..., \vx_n$ and a client sampling satisfying $\EE{S_t}{\omega_i(S_t)} = p_i$ and $\COV{\omega_i(S_t)}{\omega_j(S_t)} = - \alpha p_i p_j$. We have:
	\begin{equation}
		\EE{S_t}{\norm{\sum_{i=1}^{n}\omega_i(S_t)\vx_i}^2} 
		= \sum_{i=1}^{n} \gamma_i \norm{\vx_i}^2
		+ (1 - \alpha)\norm{\sum_{i=1}^{n}p_i \vx_i}^2,
	\end{equation}
	where $\gamma_i =  \VARR{S_t}{\omega_i(S_t)} + \alpha p_i^2 $.
\end{lemma}

\begin{proof}
	\begin{align}
		\EE{S_t}{\norm{\sum_{i=1}^{n}\omega_i(S_t)\vx_i}^2} 
		& = \sum_{i=1}^{n} \EE{S_t}{\omega_i(S_t)^2} \norm{\vx_i}^2
		+ \sum_{i=1}^{n}\sum_{\substack{j=1\\ j\neq i}}^{n}\EE{S_t}{\omega_i(S_t)\omega_j(S_t)}\inner{\vx_i}{\vx_j}.
		\label{app:eq:AA1}
	\end{align}
	
	In addition, we have:
	\begin{equation}
		\EE{S_t}{\omega_i(S_t)\omega_j(S_t)}
		= \COV{\omega_i(S_t)}{\omega_j(S_t)} + p_i p_j
		= (- \alpha + 1) p_i p_j,
		\label{app:eq:AA2}
	\end{equation}
	where the last equality comes from the assumption on the client sampling covariance.
	
	We also have:
	\begin{equation}
		\sum_{i=1}^{n}\sum_{\substack{j=1\\ j\neq i}}^{n}\inner{p_i \vx_i}{p_j \vx_j} 
		= \norm{\sum_{i=1}^{n}p_i \vx_i}^2 - \sum_{i=1}^{n}p_i^2 \norm{\vx_i}^2,
		\label{app:eq:AA3}
	\end{equation}
	
	Substituting equation (\ref{app:eq:AA2})  and equation (\ref{app:eq:AA3}) in equation (\ref{app:eq:AA1}) gives:
	\begin{equation}
		\EE{S_t}{\norm{\sum_{i=1}^{n}\omega_i(S_t)\vx_i}^2} 
		= \sum_{i=1}^{n} \left[ \EE{S_t}{\omega_i(S_t)^2} - (-\alpha +1 )p_i^2 \right] \norm{\vx_i}^2
		+ (-\alpha +  1)\norm{\sum_{i=1}^{n}p_i \vx_i}^2,
		\label{app:eq:AA4}
	\end{equation}
	
	Considering that we have $\EE{S_t}{\omega_i(S_t)^2} = \VAR{\omega_i(S_t)} + p_i^2$, we have :
	\begin{equation}
	\EE{S_t}{\omega_i(S_t)^2} + (\alpha - 1) p_i^2
	= \VARR{S_t}{\omega_i(S_t)} + \alpha p_i^2,
	\label{app:eq:AA5}
	\end{equation}
	
	Substituting equation (\ref{app:eq:AA5}) in equation (\ref{app:eq:AA4}) completes the proof.

\end{proof}

\begin{lemma}[equation (87) in \cite{FedNova}]\label{lem:dif_grads}
	Under Assumptions \ref{ass:smoothness} to \ref{ass:unbiased}, we can prove
	\begin{align}
		\frac{1}{2}\sum_{i=1}^{n} w_i \E{\norm{\nabla \Lcal_i(\vtheta^t) - \vh_i^t}^2}
		&\le \frac{1}{2} \frac{\eta_l^2 L^2 \sigma^2}{1 -R} \sum_{i=1}^{n} w_i \left( \norm{\va_i}_2^2 - a_{i, -1}^2\right)
		+ \frac{R \beta^2 }{2 (1 - R)}\E{\norm{\nabla \tilde{\Lcal}(\vtheta^t)}^2}
		+ \frac{R \kappa^2 }{2 ( 1 -R)},
	\end{align}
	with $R = 2 \eta_l^2 L^2 \max_{i}\{\norm{\va_i}_1 (\norm{\va_i}_1 - a_{i, -1})\}$ with a learning rate such that $R < 1$.
\end{lemma}

\begin{proof}
	The proof is in Section C.5 of \cite{FedNova}.
	
	The bound here provided  is slightly tighter in term of numerical constants than the one of \cite{FedNova}. Indeed, equation (70) in \cite{FedNova} uses the Jensen's inequality $\norm{ \va + \vb}^2 \le 2 \norm{\va}^2 + 2 \norm{\vb}^2$ which could instead be obtained with:
	\begin{equation}
		\E{\norm{\sum_{s=0}^{k-1}a_{i,s} g_i(\vy_{i, s}^t)}^2}
		= \E{\norm{\sum_{s=0}^{k-1}a_{i,s} \left(g_i(\vy_{i, s}^t ) - \nabla \Lcal_i(\vy_{i, s}^t) \right)}^2}
		+ \E{\norm{\sum_{s=0}^{k-1}a_{i,s} \nabla \Lcal_i(\vy_{i, s}^t)}^2},
	\end{equation}
	which uses Assumption \ref{ass:unbiased}, giving $\E{\inner{\sum_{s=0}^{k-1}a_{i,s} \left(g_i(\vy_{i, s}^t ) - \nabla \Lcal_i(\vy_{i, s}^t \right)}{\sum_{s=0}^{k-1}a_{i,s} \nabla \Lcal_i(\vy_{i, s}^t }} = 0$ with the same reasoning as for $U$ in equation (\ref{app:eq:B2}).
\end{proof}

\begin{lemma}\label{lem:contrib_client_sampling}
	Under Assumptions \ref{ass:smoothness} to \ref{ass:unbiased}, we can prove
	\begin{equation}
	\sum_{i=1}^n \gamma_i \E{\norm{ a_i \vh_i^t}^2}
	\le \frac{1}{1 - R} \sigma^2 \sum_{i=1}^{n} \gamma_i \left(\norm{\va_i}_2^2 - (a_{i, -1}^2) \right)
	+ 2  \frac{1}{1 - R} \left(\sum_{i=1}^{n}\gamma_i a_i^2\right) \left(\beta^2 \E{\norm{\nabla  \tilde{\Lcal} (\vtheta^t)}^2} + \kappa^2 \right),
	\end{equation}
	where $R' = 2 \eta_l^2 L^2 \max_i \{ \norm{a_i}_1^2 \} <1$.
	
\end{lemma}

\begin{proof}
	Due to the definition of $\vh_i^t$, we have:
	\begin{align}
		\E{\norm{ a_i \vh_i^t}^2}
		= a_i^2 \E{\norm{\sum_{k=0}^{K-1} \frac{1}{a_i} a_{i, k} \nabla \Lcal_i (\vy_{i, k}^t)}^2}
		\le a_i^2 \sum_{k=0}^{K-1}\frac{1}{a_i} a_{i, k} \E{\norm{\nabla \Lcal_i (\vy_{i, k}^t)}^2}.
		\label{app:eq:C1}
	\end{align}
	
	Using Jensen inequality, we have
	\begin{align}
		\E{\norm{\nabla \Lcal_i (\vy_{i, k}^t)}^2}
		&\le 2 \E{ \norm{\nabla \Lcal_i (\vy_{i, k}^t) - \nabla \Lcal_i (\vtheta^t)}^2}
		+ 2 \E{\norm{\nabla \Lcal_i (\vtheta^t)}^2}\\
		&\le 2 L^2 \E{ \norm{ \vy_{i, k}^t - \vtheta^t}^2}
		+ 2 \E{\norm{\nabla \Lcal_i (\vtheta^t)}^2},
		\label{app:eq:C2}
	\end{align}
	where the second equality comes from using Assumption \ref{ass:smoothness}.
	
	Also, Section C.5 of \cite{FedNova} proves
	\begin{equation}
		\frac{1}{a_i} \sum_{k=0}^{K-1} a_{i, k}\E{ \norm{ \vy_{i, k}^t - \vtheta^t}^2}
		\le \frac{1}{1 - R}\eta_l^2 \sigma^2 \left(\norm{\va_i}_2^2 - (a_{i, -1}^2) \right)
		+ \frac{1}{L^2}\frac{R}{1 - R} \E{\norm{\nabla \Lcal_i (\vtheta^t)}^2}
		.
		\label{app:eq:C3}
	\end{equation}
	
	Plugging equation (\ref{app:eq:C2}) and then equation (\ref{app:eq:C3}) in equation (\ref{app:eq:C1}), we get:
	\begin{align}
	\E{\norm{ a_i \vh_i^t}^2}
	&\le a_i^2 \sum_{k=0}^{K-1}\frac{1}{a_i} a_{i, k} \left[2 L^2 \E{ \norm{ \vy_{i, k}^t - \vtheta^t}^2}
	+ 2 \E{\norm{\nabla \Lcal_i (\vtheta^t)}^2}\right]
	\\
	&= 2 L^2  a_i^2 \sum_{k=0}^{K-1}\frac{1}{a_i} a_{i, k}  \E{ \norm{ \vy_{i, k}^t - \vtheta^t}^2}
	+ 2 a_i^2 \E{\norm{\nabla \Lcal_i (\vtheta^t)}^2}
	\\
	&\le 2 L^2  a_i^2 \left[\frac{1}{1 - R}\eta_l^2 \sigma^2 \left(\norm{\va_i}_2^2 - (a_{i, -1}^2) \right)
	+ \frac{1}{L^2}\frac{R}{1 - R} \E{\norm{\nabla \Lcal_i (\vtheta^t)}^2}\right]
	+ 2 a_i^2 \E{\norm{\nabla \Lcal_i (\vtheta^t)}^2}
	\\
	&\le \frac{R'}{1 - R} \sigma^2 \left(\norm{\va_i}_2^2 - (a_{i, -1}^2) \right)
	+ 2 a_i^2 \left[ \frac{R}{1 - R} +1 \right] \E{\norm{\nabla \Lcal_i (\vtheta^t)}^2}.
	\label{app:eq:C4}
	\end{align}
	
	Multiplying by $\gamma_i$ and summing over $n$ gives
	\begin{align}
	\sum_{i=1}^n \gamma_i \E{\norm{ a_i \vh_i^t}^2}
	&\le \frac{R'}{1 - R} \sigma^2 \sum_{i=1}^{n} \gamma_i \left(\norm{\va_i}_2^2 - (a_{i, -1}^2) \right)
	+ 2  \frac{1}{1 - R} \sum_{i=1}^{n}\gamma_i a_i^2 \E{\norm{\nabla \Lcal_i (\vtheta^t)}^2}.
	\label{app:eq:C5}
	\end{align}
	
	Using Assumption \ref{ass:dissimilarity} in equation (\ref{app:eq:C5}) and $R' < 1$ completes the proof.

%

\end{proof}

\subsection{Intermediary Theorem}\label{app:sec:intermediary_theo}

\begin{theorem}\label{theo:intermediary}
The following inequality holds:
\begin{align}
\frac{1}{T}\sum_{t=0}^{T-1} \E{\norm{\nabla \tilde{\Lcal}(\vtheta^t)}^2}
&\le \Ocal(\frac{1}{ ( 1 - \Omega) \tilde{\eta} \left(\sum_{i=1}^{n} p_i a_i\right) T}) 
+ \Ocal (\tilde{\eta} \frac{1}{m} A' \sigma^2 )
+ \Ocal( \eta_l^2 B'  \sigma^2 )
\nonumber\\
& + \Ocal (\eta_l^2 C' \kappa^2 )
+ \Ocal(\tilde{\eta} D' \sigma^2)
+ \Ocal(\tilde{\eta} E' \kappa^2)
,
\end{align}
where quantities $A'$-$E'$ are defined in the following proof from equation (\ref{app:eq:def_A}) to equation (\ref{app:eq:def_E}).
\end{theorem}

\begin{proof}

Clients local loss functions are $L$-Lipschitz smooth. Therefore, $\tilde{\Lcal}$ is also $L$-Lipschitz smooth which gives
\begin{equation}
	\E{\tilde{\Lcal}(\vtheta^{t+1}) - \tilde{\Lcal}(\vtheta^t)}
	\le \underbrace{\E{\inner{\nabla \tilde{\Lcal}(\vtheta^t)}{\vtheta^{t+1} - \vtheta^t}}}_{T_1}
	+ \frac{L}{2}\underbrace{\E{\norm{\vtheta^{t+1} - \vtheta^t }^2}}_{T_2}
,
\label{app:eq:B1}
\end{equation}
where the expectation is taken over the subset of randomly sampled clients $S_t$ and the clients gradient estimator noises $\xi_i^t$. Please note that we use the notation $\E{\cdot}$ instead of $\EE{\{\xi_i^t\}, S_t}{\cdot}$ for ease of writing.

\subsubsection*{Bounding $T_1$}
By conditioning on $\{\xi_i^t \}$ and using equation (\ref{app:eq:expected_client_contrib}), we get:
\begin{equation}
	T_1 
	= \E{\inner{\nabla \tilde{\Lcal}(\vtheta^t)}{{\EE{S_t}{\vtheta^{t+1} - \vtheta^t}}}} 
	= - \tilde{\eta} K_{eff} \E{\inner{\nabla \tilde{\Lcal}(\vtheta^t)}{ \sum_{i=1}^{n} w_i \vh_i^t}},
\end{equation}
which, using $2 \inner{a}{b} = \norm{a}^2 + \norm{b}^2 - \norm{a - b}^2$ can be rewritten as:
\begin{equation}
T_1 
= - \frac{1}{2}\tilde{\eta} K_{eff} \E{ \norm{\nabla \tilde{\Lcal}(\vtheta^t)}^2 
+ \norm{\sum_{i=1}^{n} w_i \vh_i^t}^2 
- \norm{\nabla \tilde{\Lcal}(\vtheta^t) -  \sum_{i=1}^{n} w_i \vh_i^t}^2 }.
\label{app:eq:BB1}
\end{equation}

\subsubsection*{Bounding $T_2$}

\begin{align}
	T_2 | S_t
	& = \tilde{\eta}^2 \E{\norm{\sum_{i=1}^{n} \omega_i a_i \vd_i^t}^2|S_t}
	\\
	& = \tilde{\eta}^2 \E{\norm{\sum_{i=1}^{n} \omega_i a_i \left(\vd_i^t - \vh_i^t\right) 
			+ \sum_{i=1}^{n} \omega_i a_i \vh_i^t}^2|S_t}
	\\
	& = \tilde{\eta}^2 \E{\norm{\sum_{i=1}^{n} \omega_i a_i \left(\vd_i^t - \vh_i^t\right) }^2 |S_t }
	+\tilde{\eta}^2 \E{\norm{\sum_{i=1}^{n} \omega_i a_i \vh_i^t}^2|S_t}
	\nonumber\\
	& + 2 \tilde{\eta} \underbrace{\E{\inner{\sum_{i=1}^{n} \omega_i a_i \left(\vd_i^t - \vh_i^t\right)}{\sum_{i=1}^{n} \omega_i a_i \vh_i^t}|S_t}}_U
	.
	\label{app:eq:B2}
\end{align}
Using Assumption \ref{ass:unbiased}, we have $\E{\inner{d_i^t - h_i^t}{h_j^t}} = 0$. Hence, we get $U=0$ and can simplify $T_2$ as: 
\begin{align}
T_2 
& = \tilde{\eta}^2 \sum_{i=1}^{n} \E{\omega_i^2} a_i^2\E{\norm{ \vd_i^t - \vh_i^t }^2 }
+\tilde{\eta}^2 \E{\norm{\sum_{i=1}^{n} \omega_i a_i \vh_i^t}^2}
.
\label{app:eq:B3}
\end{align}

Using Lemma \ref{lem:decompo_Xi} on the second term, we get:
\begin{align}
T_2 
& = \tilde{\eta}^2 \sum_{i=1}^{n} \E{\omega_i^2} a_i^2\E{\norm{ \vd_i^t - \vh_i^t }^2 }
+\tilde{\eta}^2 \sum_{i=1}^{n} \gamma_i \E{\norm{ a_i \vh_i^t}^2}
+\tilde{\eta}^2 (1- \alpha)\E{\norm{\sum_{i=1}^{n} p_i a_i \vh_i^t}^2}
.
\label{app:eq:B4}
\end{align}

Finally, by bounding the first term using Assumption \ref{ass:unbiased}, and noting that $p_i a_i = w_i K_{eff}$ for the second term, we get:
\begin{align}
T_2 
& = \tilde{\eta}^2 \sum_{i=1}^{n} \E{\omega_i^2} \sum_{k=0}^{K-1} a_{i, k}^2 \E{\norm{ g_i(\vy_{i, k}^t) - \nabla \Lcal_i(\vy_{i, k}^t) }^2 }
\nonumber\\
&+\tilde{\eta}^2 \sum_{i=1}^{n} \gamma_i \E{\norm{ a_i \vh_i^t}^2}
+\tilde{\eta}^2 (1- \alpha)K_{eff}^2\E{\norm{\sum_{i=1}^{n} w_i \vh_i^t}^2}
\\
& \le \tilde{\eta}^2 \sum_{i=1}^{n} \E{\omega_i^2} \norm{\va_i}_2^2 \sigma^2
+\tilde{\eta}^2 \sum_{i=1}^{n} \gamma_i \E{\norm{ a_i \vh_i^t}^2}
+\tilde{\eta}^2 (1- \alpha)K_{eff}^2\E{\norm{\sum_{i=1}^{n} w_i \vh_i^t}^2}
.
\label{app:eq:B5}
\end{align}

\subsubsection*{Going back to equation (\ref{app:eq:B1})}

Substituting equation (\ref{app:eq:BB1}) and equation (\ref{app:eq:B5}) back in equation (\ref{app:eq:B1}), we get:
\begin{align}
	\E{\tilde{\Lcal}(\vtheta^{t+1}) - \tilde{\Lcal}(\vtheta^t)}
	&\le - \frac{1}{2}\tilde{\eta} K_{eff} \norm{\nabla \tilde{\Lcal}(\vtheta^t)}^2 
	+ \frac{1}{2}\tilde{\eta} K_{eff} \E{\norm{\nabla \tilde{\Lcal}(\vtheta^t) -  \sum_{i=1}^{n} w_i \vh_i^t}^2}
	\nonumber\\
	& - \frac{1}{2}\tilde{\eta} K_{eff} \left[1 - L\tilde{\eta}(1-\alpha) K_{eff}\right] \E{ \norm{\sum_{i=1}^{n} w_i \vh_i^t}^2 }
	\nonumber\\
	&+\frac{L}{2} \tilde{\eta}^2 \sum_{i=1}^{n} \E{\omega_i^2} \norm{\va_i}_2^2 \sigma^2
	+ \frac{L}{2} \tilde{\eta}^2 \sum_{i=1}^{n} \gamma_i \E{\norm{ a_i \vh_i^t}^2}
	,
	\label{app:eq:B55}
\end{align}

We consider the learning rate to satisfy $1 - L\tilde{\eta}(1-\alpha) K_{eff}>0$ such that we can simplify equation (\ref{app:eq:B55}) as :

\begin{align}
\frac{\E{\tilde{\Lcal}(\vtheta^{t+1}) - \tilde{\Lcal}(\vtheta^t)}}{\tilde{\eta} K_{eff}}
&\le - \frac{1}{2}\norm{\nabla \tilde{\Lcal}(\vtheta^t)}^2 
+ \frac{1}{2}\E{\norm{\nabla \tilde{\Lcal}(\vtheta^t) -  \sum_{i=1}^{n} w_i \vh_i^t}^2}
\nonumber\\
&+\frac{L}{2} \tilde{\eta} \frac{1}{K_{eff}} \sum_{i=1}^{n} \E{\omega_i^2} \norm{\va_i}_2^2 \sigma^2
+ \frac{L}{2} \tilde{\eta} \frac{1}{K_{eff}} \sum_{i=1}^{n} \gamma_i \E{\norm{ a_i \vh_i^t}^2}
\\
&\le - \frac{1}{2}\norm{\nabla \tilde{\Lcal}(\vtheta^t)}^2 
+ \frac{1}{2} \sum_{i=1}^{n} w_i \E{ \norm{\nabla \Lcal_i(\vtheta^t) -   \vh_i^t}^2}
\nonumber\\
&+\frac{L}{2} \tilde{\eta} \frac{1}{K_{eff}} \sum_{i=1}^{n} \E{\omega_i^2} \norm{\va_i}_2^2 \sigma^2
+ \frac{L}{2} \tilde{\eta} \frac{1}{K_{eff}} \sum_{i=1}^{n} \gamma_i \E{\norm{ a_i \vh_i^t}^2}
,
\label{app:eq:B6}
\end{align}
where the last inequality uses the definition of the surrogate loss function $\tilde{\Lcal}$ and the Jensen's inequality.

Using Lemma \ref{lem:contrib_client_sampling} and \ref{lem:dif_grads}, we get:
\begin{align}
\frac{\E{\tilde{\Lcal}(\vtheta^{t+1}) - \tilde{\Lcal}(\vtheta^t)}}{\tilde{\eta} K_{eff}}
&\le - \frac{1}{2}\norm{\nabla \tilde{\Lcal}(\vtheta^t)}^2 
+ \frac{1}{2} \frac{\eta_l^2 L^2 \sigma^2}{1 -R} \sum_{i=1}^{n} w_i \left( \norm{\va_i}_2^2 - a_{i, -1}^2\right)
\nonumber\\
&+ \frac{R \beta^2 }{2 (1 - R)}\E{\norm{\nabla \tilde{\Lcal}(\vtheta^t)}^2}
+ \frac{R \kappa^2 }{2 ( 1 -R)}
\nonumber\\
&+\frac{L}{2} \tilde{\eta} \frac{1}{K_{eff}} \left[\sum_{i=1}^{n} \E{\omega_i^2} \norm{\va_i}_2^2 
+  \frac{1}{1 - R}  \sum_{i=1}^{n} \gamma_i \left(\norm{\va_i}_2^2 - (a_{i, -1}^2) \right)\right] \sigma^2
\nonumber\\
&+ L \tilde{\eta} \frac{1}{K_{eff}}   \left[ \frac{R}{1 - R} +1 \right] \left(\sum_{i=1}^{n}\gamma_i a_i^2\right) \left(\beta^2 \E{\norm{\nabla \tilde{\Lcal} (\vtheta^t)}^2} + \kappa^2 \right)
\label{app:eq:B66}
.
\end{align}

If we assume that $R \le \frac{1}{2 \beta^2 +1}$, and considering that $\beta^2 \ge 1$, then we have $\frac{1}{1 - R} \le 1 + \frac{1}{2 \beta^2}\le \frac{3}{2}$, $\frac{R}{1-R} \le \frac{1}{2}$ , 
and $\frac{R \beta^2}{1 -R} \le \frac{1}{2 \beta^2 + 1} (1 + \frac{1}{2 \beta^2})\beta^2 = \frac{1}{2}$.
We also define $\Omega = L \tilde{\eta} \frac{1}{K_{eff}}   \frac{3}{2} \left(\sum_{i=1}^{n}\gamma_i a_i^2\right)\beta^2 \le \frac{1}{2}$. Substituting these terms in equation (\ref{app:eq:B66}) gives

\begin{align}
\frac{\E{\tilde{\Lcal}(\vtheta^{t+1}) - \tilde{\Lcal}(\vtheta^t)}}{\tilde{\eta} K_{eff}}
&\le - \frac{1}{4}\left[1 - \Omega \right]\norm{\nabla \tilde{\Lcal}(\vtheta^t)}^2 
+ \frac{3}{4} \eta_l^2 L^2\sum_{i=1}^{n} w_i \left( \norm{\va_i}_2^2 - a_{i, -1}^2\right) \sigma^2 
\nonumber\\
&+\frac{L}{2} \tilde{\eta} \frac{1}{K_{eff}} \left[\sum_{i=1}^{n} \E{\omega_i^2} \norm{\va_i}_2^2 
+ \frac{3}{2}\sum_{i=1}^{n} \gamma_i \left(\norm{\va_i}_2^2 - (a_{i, -1}^2) \right)\right] \sigma^2
\nonumber\\
&+ \frac{3}{2}\eta_l^2 L^2\max_i\{a_i (a_i - a_{i, -1})\} \kappa^2 
+ \frac{3}{2}L \tilde{\eta} \frac{1}{K_{eff}}   \left(\sum_{i=1}^{n}\gamma_i a_i^2\right)  \kappa^2
.
\end{align}

%
%

Averaging across all rounds, we get:
\begin{align}
\frac{1 - \Omega}{T}\sum_{t=0}^{T-1} \E{\norm{\nabla \tilde{\Lcal}(\vtheta^t)}^2}
&\le 4   \frac{\tilde{\Lcal}(\theta^0) - \tilde{\Lcal}(\theta^*)}{\tilde{\eta} K_{eff}T} 
+ 3 \eta_l^2 L^2 \sum_{i=1}^{n} w_i \left( \norm{\va_i}_2^2 - a_{i, -1}^2\right) \sigma^2
\nonumber\\
&+ L \tilde{\eta} \frac{1}{K_{eff}} \left[2 \sum_{i=1}^{n} \E{\omega_i^2} \norm{\va_i}_2^2
+ 3  \sum_{i=1}^{n} \gamma_i \left(\norm{\va_i}_2^2 - (a_{i, -1}^2) \right)\right] \sigma^2
\nonumber\\
&+ 6\eta_l^2 L^2\max_i\{a_i (a_i - a_{i, -1})\} \kappa^2 
+ 6 L \tilde{\eta} \frac{1}{K_{eff}}   \left(\sum_{i=1}^{n}\gamma_i a_i^2\right)  \kappa^2.
\end{align}

We define the following auxiliary variables
\begin{equation}
	A 
	= m \frac{1}{K_{eff}} \sum_{i=1}^{n} \E{\omega_i^2} \norm{\va_i}_2^2 
	= m \frac{1}{\sum_{i=1}^{n}p_i a_i} \sum_{i=1}^{n} \left[\VAR{\omega_i} + p_i^2\right] \norm{\va_i}_2^2 ,
	\label{app:eq:def_A}
\end{equation}
\begin{equation}
B
= \sum_{i=1}^{n} w_i \left( \norm{\va_i}_2^2 - a_{i, -1}^2\right)
= \sum_{i=1}^{n}\frac{p_i a_i}{\sum_{j=1}^{n} p_j a _j}  \left( \norm{\va_i}_2^2 - a_{i, -1}^2\right)
,
\end{equation}
\begin{equation}
C
= \max_i\{a_i (a_i - a_{i, -1})\}
,
\end{equation}
\begin{equation}
D 
= \frac{1}{K_{eff}} \max_i\{a_i (a_i - a_{i, -1})\} \sum_{i=1}^{n}\gamma_i  
= \frac{1}{\sum_{i=1}^{n}p_i a_i} C \left( \sum_{i=1}^{n} \VAR{\omega_i} + \alpha \sum_{i=1}^{n}p_i^2 \right)
,
\end{equation}
\begin{equation}
E 
= \frac{1}{K_{eff}}\max_{i} \{a_i^2\}\left(\sum_{i=1}^{n}\gamma_i \right) 
= \frac{1}{\sum_{i=1}^{n}p_i a_i} \max_{i} \{a_i^2\}\left(\sum_{i=1}^{n} \VAR{\omega_i} + \alpha \sum_{i=1}^{n} p_i^2 \right)
\label{app:eq:def_E}
.
\end{equation}


We define for $A$ -$E$ the respective quantities $A'$-$E'$ such that $X' = \frac{1}{1 -\Omega} X$. We have:
\begin{align}
\frac{1}{T}\sum_{t=0}^{T-1} \E{\norm{\nabla \tilde{\Lcal}(\vtheta^t)}^2}
&\le 4   \frac{\tilde{\Lcal}(\theta^0) - \tilde{\Lcal}(\theta^*)}{( 1- \Omega) \tilde{\eta} \left(\sum_{i=1}^{n} p_i a_i\right) T} 
+ 2L \tilde{\eta} \frac{1}{m} A' \sigma^2
+ 3 \eta_l^2 L^2 B'  \sigma^2 
\nonumber\\
& + 6\eta_l^2 L^2 C' \kappa^2 
+ 3L \tilde{\eta} D\sigma^2
+ 6 L \tilde{\eta} E \kappa^2
,
\label{app:eq:end_intermediary}
\end{align}

\end{proof}

\subsection{Synthesis of local learning rate $\eta_l$ conditions for Theorem \ref{theo:intermediary}}

{
A sufficient bound on the local learning rate $\eta_l$ for constraints on $R$ for Lemma \ref{lem:dif_grads} and equation (\ref{app:eq:B66}), and constraint on $R'$ for Lemma \ref{lem:contrib_client_sampling}  to be satisfied is:
\begin{equation}
    2 \left[ 2 \beta^2 + 1\right]\eta_l^2 L^2 \max_i \{ \norm{a_i}_1^2 \} <1.
\end{equation}
Constraints on equation (\ref{app:eq:B55}) can be simplified as
\begin{equation}
    L \eta_g \eta_l (1-\alpha) K_{eff}< 1.
\end{equation}
    
Constraints on $\Omega$, equation (\ref{app:eq:B66}), give
\begin{equation}
    3 L \eta_g \eta_l \frac{1}{K_{eff}}  \left(\sum_{i=1}^{n}\gamma_i a_i^2\right)\beta^2 \le 1.
\end{equation}
}

\subsection{Theorem \ref{theo:convergence_paper}}\label{app:sec:proof_convergence}

\begin{proof}
With \textsc{FedAvg}, every client performs vanilla SGD. As such, we have $a_{i, k} = 1$ which gives $a_i = K$ and $\norm{a_i}_2 = \sqrt{K}$. In addition we consider a local learning rate $\eta_l$ such that $\Omega \le \frac{1}{2}$ as such we can bound $A'$-$E'$ as $ X' \le 2 X$. 

Finally, considering that the variables $A$ to $E$ can be simplified as
\begin{equation}
A 
= m \sum_{i=1}^{n} \left[\VAR{\omega_i} + p_i^2\right]
,
B
= (K -1)
,
C
= K (K-1)
,
\end{equation}
\begin{equation}
D 
= (K - 1) \left( \sum_{i=1}^{n} \VAR{\omega_i} + \alpha \sum_{i=1}^{n} p_i^2  \right)
,
\text{ and }
E 
= K  \left( \sum_{i=1}^{n} \VAR{\omega_i} + \alpha \sum_{i=1}^{n} p_i^2 \right),
\end{equation}
the convergence bound of Theorem \ref{theo:intermediary} can be reduced to
\begin{align}
\frac{1}{T}\sum_{t=0}^{T-1} \E{\norm{\nabla \Lcal(\vtheta^t)}^2}
&\le \Ocal \left( \frac{1}{\eta_g \eta_l K T} \right)
+  \Ocal \left( \eta_g \eta_l \sum_{i=1}^{n} \left[\VAR{\omega_i} + p_i^2\right] \sigma^2 \right)
\nonumber\\
&+ \Ocal \left(  \eta_l^2 (K-1) \sigma^2 \right)
+ \Ocal \left( \eta_l^2 K(K-1) \kappa^2 \right)
\nonumber\\
& + \Ocal \left( \eta_g \eta_l \left(\sum_{i=1}^{n}\VAR{\omega_i} + \alpha \sum_{i=1}^{n} p_i^2  \right) \left[(K-1)\sigma^2 + K \kappa^2 \right] \right)
\label{app:eq:F1}
,
\end{align}
which completes the proof.

\end{proof}    

$\Omega$ is proportional to $\sum_{i=1}^{n}\gamma_i = \sum_{i=1}^{n}\VAR{\omega_i} + \alpha \sum_{i=1}^{n} p_i^2 $. With full participation, we have $\Omega = 0$. However, with client sampling, all the terms in equation (\ref{app:eq:F1}) are proportional with $\frac{1}{1 - \Omega}$. Yet, we provide a looser bound in equation (\ref{app:eq:F1}) independent from $\Omega$ as the conclusions drawn are identical. Through $\Omega$, $\sum_{i=1}^{n}\VAR{\omega_i}$ and  $\alpha$ needs to be minimized. This fact is already visible by inspection of the quantities $E$ and $F$.

We note that equation (\ref{app:eq:F1}) depends on client sampling through $\sigma^2$, which is an indicator of the clients SGD quality, and $\kappa^2$, which depends on the clients data heterogeneity. In the special case where clients have the same data distribution and perform full gradient descent, based on the arguments discussed in the previous paragraph, we can still provide the following bound showing the influence of client sampling on the convergence speed, while highlighting the interest of minimizing the quantities $\sum_{i=1}^{n}\VAR{\omega_i}$ and  $\alpha$.
\begin{align}
\frac{1}{T}\sum_{t=0}^{T-1} \E{\norm{\nabla \Lcal(\vtheta^t)}^2}
&\le \Ocal \left( \frac{1}{(1 - \Omega)\eta_g \eta_l K T} \right)
,
\end{align}

When setting the server learning rate at 1, $\eta_g=1$ with client full participation, i.e. $\VAR{\omega_i} = \VAR{\sum_{i=1}^{n}\omega_i} = \alpha = 0$ and $m=n$, we have $E= F=0$ and can simplify $A$ to                          
\begin{equation}                                                                                                         
A
= n \sum_{i=1}^{n} p_i^2                                                                                                                                                                                                              
.                                                                                                                                                                                                                                     
\end{equation}                                                                                                                                                                                                                        
Therefore, the convergence guarantee we provide is $\frac{1}{\eta_lKT} + \eta_l \sum_{i=1}^{n}p_i^2 \sigma^2 + \eta_l^2 (K-1) \sigma^2 + \eta_l^2 K(K-1) \kappa^2 $, which is identical to the one of \cite{FedNova} (equation (97) in their work), where $\sum_{i=1}^{n}p_i^2$ can be replaced by $1/n$ when clients have identical importance, i.e. $p_i = 1/n$.

In the special case, where we use $\eta_l = \sqrt{m/KT}$ \citep{FedNova}, we retrieve their asymptotic convergence bound $\frac{1}{\sqrt{mKT}} + \sqrt{\frac{m}{KT}} \sum_{i=1}^{n}p_i^2 \sigma^2 + \frac{m}{T} \sigma^2 + \frac{m}{T} K \kappa^2 $.

\subsection{Application to Clustered Sampling}\label{app:sec:clustered_convergence}

Instead of Lemma \ref{lem:decompo_Xi} which requires $\COV{\omega_i(S_t)}{\omega_j(S_t)} = - \alpha p_i p_j$, we propose the following Lemma for Clustered sampling expressed in function of MD sampling covariance parameter $\alpha_{MD}$ showing that a sufficient condition for MD sampling to perform as well as Clustered sampling is that all $\vx_i$ are identical, or that all the distributions are identical, i.e. $r_{k, i} = p_i$. 

\begin{lemma}\label{lem:decompo_Xiclustered}
	Let us consider $n$ vectors $\vx_i, ..., \vx_n$ and a Clustered sampling satisfying $\EE{S_t}{\omega_i(S_t)} = p_i$. We have:
	\begin{equation}
	\EE{S_{Cl}}{\norm{\sum_{i=1}^{n}\omega_i(S_{Cl})\vx_i}^2} 
	\le \sum_{i=1}^{n} \gamma_i(MD) \norm{\vx_i}^2
	+ (1 - \alpha_{MD})\norm{\sum_{i=1}^{n}p_i \vx_i}^2,
	\label{eq:decompo_clustered}
	\end{equation}
	where $\gamma_i(MD)$ and $\alpha_{MD}$ are the aggregation weights statistics of MD sampling. Equation (\ref{eq:decompo_clustered}) is an equality if and only if $\sum_{i=1}^{n} r_{k, i} \vx_i = \sum_{j=1}^{n} r_{k, j} \vx_j$.
\end{lemma}

\begin{proof}
	Substituting equation (\ref{app:eq:T1}) in equation (\ref{app:eq:AA1}) gives
	\begin{align}
	\EE{S_{Cl}}{\norm{\sum_{i=1}^{n}\omega_i(S_{Cl})\vx_i}^2} 
	& = \sum_{i=1}^{n} \EE{ S_{Cl} }{\omega_i( S_{Cl} )^2} \norm{\vx_i}^2
	+ \sum_{i=1}^{n}\sum_{\substack{j=1\\ j\neq i}}^{n} p_i p_j \inner{\vx_i}{\vx_j}
	- \frac{1}{m^2} \sum_{k=1}^{m} \sum_{i=1}^{n}\sum_{\substack{j=1\\ j\neq i}}^{n}  r_{k, i}r_{k, j} \inner{\vx_i}{\vx_j},
	\label{app:eq:U1}
	\end{align}
	
	Substituting equation (\ref{app:eq:AA3}) in equation (\ref{app:eq:AA1}) gives:
	\begin{align}
	\EE{ S_{Cl} }{\norm{\sum_{i=1}^{n}\omega_i( S_{Cl} )\vx_i}^2} 
	& = \sum_{i=1}^{n} \EE{ S_{Cl} }{\omega_i( S_{Cl} )^2} \norm{\vx_i}^2
	+ \norm{\sum_{i=1}^{n}p_i \vx_i }^2 
	- \sum_{i=1}^{n}p_i^2 \norm{ \vx_i }^2 
	\nonumber\\
	&- \frac{1}{m^2} \sum_{k=1}^m \left[\norm{\sum_{i=1}^{n}r_{k, i} \vx_i }^2 
	- \sum_{i=1}^{n} r_{k, i}^2 \norm{ \vx_i }^2 \right]
	\label{app:eq:U2}.
	\end{align}
	
	With rearrangements and using equation (\ref{app:eq:clustered_prop}) we get:
	\begin{align}
	\EE{ S_{Cl} }{\norm{\sum_{i=1}^{n}\omega_i( S_{Cl} )\vx_i}^2} 
	& = \sum_{i=1}^{n} \left[ \VAR{\omega_i( S_{Cl} )} + \frac{1}{m^2}\sum_{k=1}^m r_{k, i}^2 \right] \norm{\vx_i}^2
	+ \norm{\sum_{i=1}^{n}p_i \vx_i }^2 
    - \frac{1}{m^2} \sum_{k=1}^m \norm{\sum_{i=1}^{n}r_{k, i} \vx_i }^2. 
	\label{app:eq:U3}
	\end{align}
	
	Using the expression of clustered sampling variance for the first term (equation (\ref{app:eq:T2})),  and using Jensen's inequality on the third term completes the proof. Jensen's inequality is an equality if and only if $\sum_{i=1}^{n} r_{k, i} \vx_i = \sum_{j=1}^{n} r_{k, j} \vx_j$.
	
\end{proof}

We adapt Theorem \ref{theo:convergence_paper} to Clustered sampling. \cite{ClusteredSampling} prove the convergence of FL with clustered sampling by giving identical convergence guarantees to the one of FL with MD sampling. As a result, their convergence bound does not depend of the clients selection probability in the different clusters $r_{k, i}$. The authors' claim was that reducing the variance of the aggregation weights provides faster FL convergence, albeit only providing experimental proofs was provided to support this statement. Corollary \ref{cor:clustered} here proposed extends the theory of  \cite{ClusteredSampling} by theoretically demonstrating the influence of clustered sampling on the convergence rate. For easing the notation, Corollary \ref{cor:clustered} is adapted to \textsc{FedAvg} but can easily be extended to account for any local $\va_i$ using the proof of Theorem \ref{theo:intermediary} in Section \ref{app:sec:intermediary_theo}.

\begin{corollary}\label{cor:clustered}
	Even with no $\alpha$ such that $\COV{\omega_i(S_t)}{\omega_j(S_t)} = - \alpha p_i p_j$, the bound of Theorem \ref{theo:convergence_paper} still holds with $B$, $C$, and $D$ defined as in Section \ref{app:sec:intermediary_theo} and 
	\begin{equation}
	A = m \left[\frac{1}{m}
	- \frac{1}{m^2} \sum_{i=1}^{n} \sum_{k=1}^{m}r_{k, i}^2 + \sum_{i=1}^{n}p_i^2\right]
	,\
	E 
	= \frac{1}{m}(K-1)
	,
	\text{ and }
	F
	= \frac{1}{m}K
	,
	\end{equation}
	where $E$ and $F$ are identical to the one for MD sampling and $A$ is smaller than the one for Clustered sampling.
\end{corollary}

\begin{proof}
	The covariance property required for Theorem \ref{theo:intermediary} is only used for Lemma \ref{lem:decompo_Xi}. In the proof of Theorem \ref{theo:intermediary}, Lemma \ref{lem:decompo_Xi} is only used in equation (\ref{app:eq:B4}). We can instead use Lemma \ref{lem:decompo_Xiclustered} and keep the rest of the proof as it is in Section \ref{app:sec:intermediary_theo}. Therefore, the bound of Theorem \ref{theo:intermediary} remains unchanged for clustered sampling where $E$ and $F$ use the aggregation weight statistics of MD sampling instead of clustered sampling. Statistics for MD sampling can be found in Section \ref{app:sec:MD} and give
	\begin{equation}
	\VAR{\sum_{i=1}^{n}\omega_i(S_{MD})} =0
	\text{ and }
	\alpha_{MD} = \frac{1}{m}, 
	\end{equation}
	while the ones of clustered sampling in Section \ref{app:sec:clustered} give
	\begin{equation}
	\sum_{i=1}^{n} \VAR{\omega_i(S_{Cl})} 
	= \frac{1}{m}
	- \frac{1}{m^2} \sum_{i=1}^{n} \sum_{k=1}^{m}r_{k, i}^2
	\le \sum_{i=1}^{n} \VAR{\omega_i(S_{MD})} 
	.
	\end{equation}
	
\end{proof}


\subsection{ Proof of Corollary \ref{cor:sufficient_conditions}}\label{app:sec:sufficient_conditions} 

\begin{proof}
	
	Combining equation (\ref{app:eq:R1}) with equation (\ref{app:eq:S2}) gives
	\begin{align}
		\Sigma_{MD}
		- \Sigma_U
		= \left[ - \frac{1}{m} \sum_{i=1}^{n}p_i^2 + \frac{1}{m} \right]
		- \left( \frac{n}{m} - 1 \right) \sum_{i=1}^{n}p_i^2
		= - \frac{1}{m}\left[(n- m +1) \sum_{i=1}^{n} p_i^2 -1 \right]
		.
		\label{app:eq:N1}
	\end{align}
	Therefore, we have
	\begin{equation}
		\Sigma_{MD} \le \Sigma_U
		\Leftrightarrow
		\sum_{i=1}^{n}p_i^2 
		\le 
		\frac{1}{n - m +1}.
	\end{equation}
	
	Combining equation (\ref{app:eq:R2}), (\ref{app:eq:R3}), (\ref{app:eq:S5}), and (\ref{app:eq:S6}) gives
	\begin{align}
	\gamma_{MD} - \gamma_{U}
	= \sum_{i=1}^{n} \VAR{  \omega_i(S_{MD})} + \alpha_{MD} \sum_{i=1}^{n} p_i^2
	- \left( \sum_{i=1}^{n}   \VAR{ \omega_i(S_{U})} + \alpha_U \sum_{i=1}^{n}  p_i^2 \right)
	= \frac{1}{m} - \frac{n-m}{m (n-1)} n \sum_{i=1}^{n}p_i^2
	.
	\label{app:eq:N2}
	\end{align}
	
	Therefore, we have
	\begin{equation}
	\gamma_{MD} \le \gamma_U
	\Leftrightarrow
	\sum_{i=1}^{n}p_i^2 
	\le 
	\frac{1}{n-m} \frac{n-1}{n}
	.
	\end{equation}

	Noting that 
	\begin{equation}
		\frac{1}{n - m +1} - \frac{1}{n-m} \frac{n-1}{n} 
		= \frac{-m +1}{n (n-m )(n-m+1)}
		\le 0,
	\end{equation}
	completes the proof.

\end{proof}

\section{Additional experiments}\label{app:sec:additional_experiments}

\subsection{Shakespeare dataset}

The client local learning rate $\eta_l$ is selected in \{0.1, 0.5, 1., 1.5, 2., 2.5\} minimizing \textsc{FedAvg} with full participation, and $n = 80$ training loss at the end of the learning process.


\begin{figure}[H]
\begin{center}
	\includegraphics[width = 0.8\textwidth]{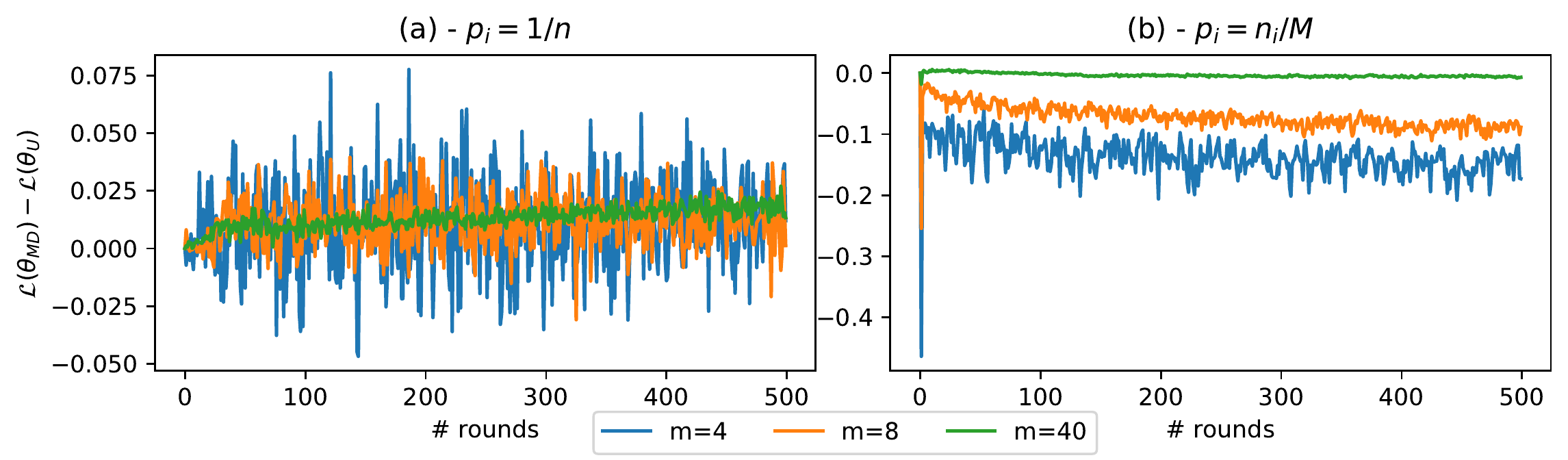}
	
\end{center}
	\caption{Difference between the convergence of the global losses resulting from MD and Uniform sampling when considering $n = 80$ clients and sampling $m \in \{4, 8, 40\}$ of them while clients perform $K=50$ SGD steps . In (a), clients have identical importance, i.e. $p_i = 1/n$. In (b), clients importance is proportional to their amount of data, i.e. $p_i = n_i / M$. Differences in global losses are averaged across 15 FL experiments with different model initialization (global losses are provided in Figure \ref{fig:Shak_diff_m}). }
	\label{fig:Shak_diff_m_synthesis}
\end{figure}

\begin{figure}[H]
    \begin{center}
	\includegraphics[width = 0.8\textwidth]{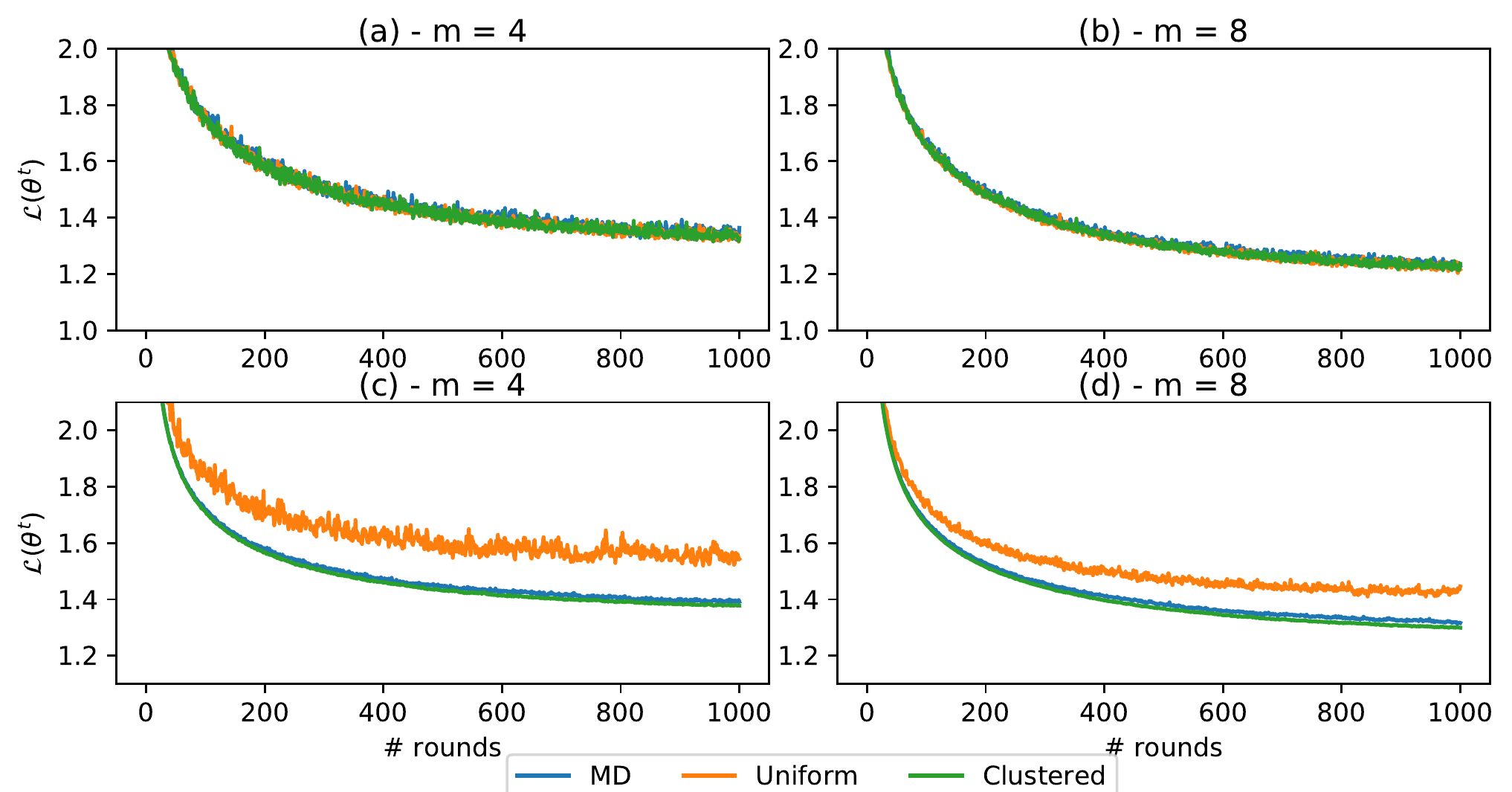}
	\end{center}
	\caption{Convergence speed of the global loss with MD sampling and Uniform sampling when considering $n=80$  clients while sampling $m=4$ ((a) and (c)), and $m=8$ ((b) and (d)) while clients perform $K=50$ SGD steps. In (a-b) , clients have identical importance, i.e. $p_i = 1/n$, and, in (d-f), their importance is proportional to their amount of data, i.e. $p_i = n_i / M$. Global losses are estimated on 15 different model initialization. }
	\label{fig:Shak_diff_m}
\end{figure}

\begin{figure}[H]
    \begin{center}
	\includegraphics[width = 0.8\textwidth]{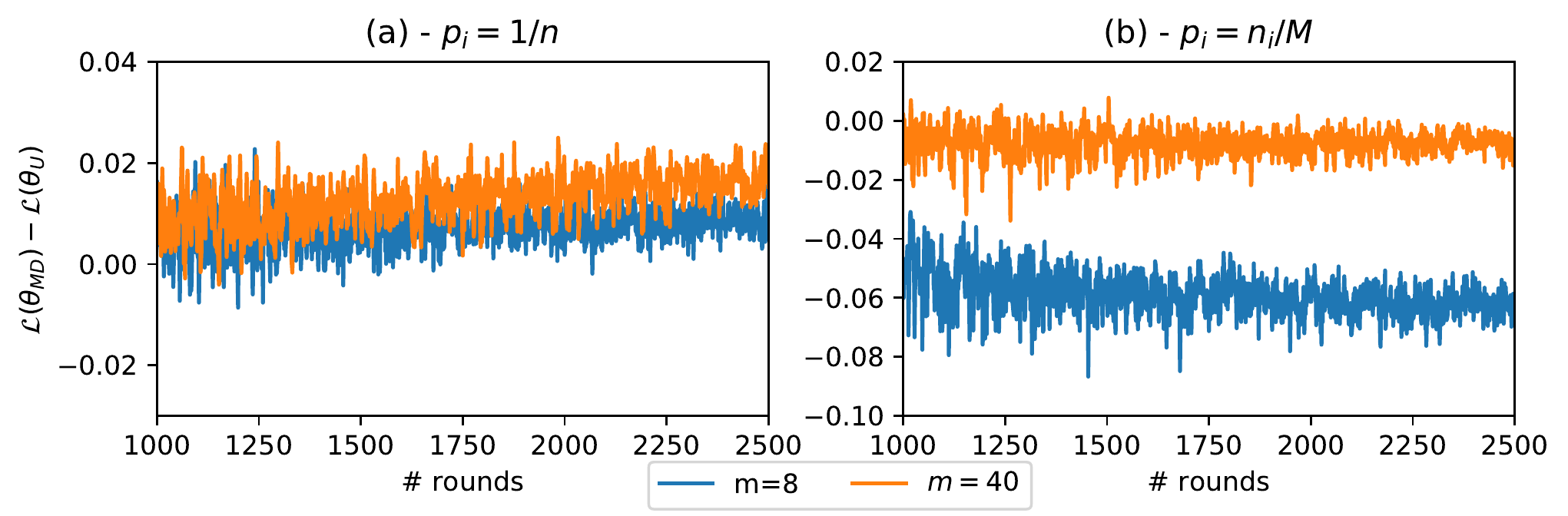}
	\end{center}
	\caption{Difference between the convergence of the global losses resulting from MD and Uniform sampling when considering $n = 80$ clients and sampling $m \in \{8, 40\}$ of them while clients perform $K=1$ SGD step. In (a), clients have identical importance, i.e. $p_i = 1/n$. In (b), clients importance is proportional to their amount of data, i.e. $p_i = n_i / M$. Differences in global losses are averaged across 15 FL experiments with different model initialization (global losses are provided in Figure \ref{fig:Shak_diff_K}). }
	\label{fig:Shak_diff_K_synthesis}
\end{figure}

\begin{figure}[H]
    \begin{center}
	\includegraphics[width = 0.65 \textwidth]{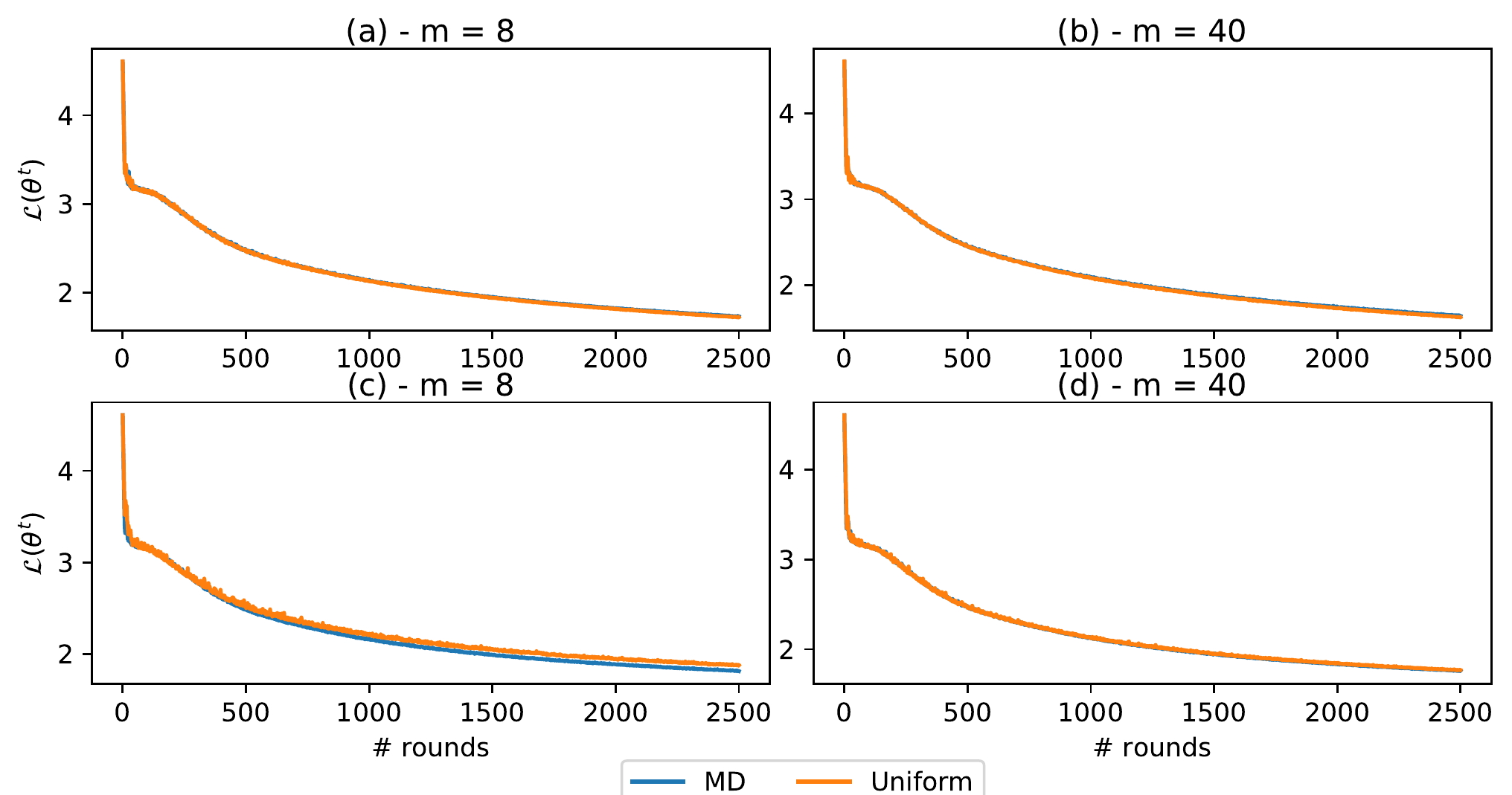}
	\end{center}
	\caption{Convergence speed of the global loss with MD sampling and Uniform sampling when considering $n=80$  clients while sampling $m=4$ ((a) and (d)), $m=8$ ((b) and (e)), $m=40$ ((c) and (f)) while clients perform $K=1$ SGD steps. In (a-c) , clients have identical importance, i.e. $p_i = 1/n$, and, in (d-f), their importance is proportional to their amount of data, i.e. $p_i = n_i / M$. Global losses are estimated on 15 different model initialization. }
	\label{fig:Shak_diff_K}
\end{figure}

\subsection{CIFAR10 dataset}

We consider the experimental scenario used to prove the experimental correctness of clustered sampling in \citep{ClusteredSampling} on CIFAR10 \citep{CIFAR-10}. The dataset is partitioned in $n=100$ clients using a Dirichlet distribution with parameter $\alpha =0.1$ as proposed in \cite{FL_and_CIFAR_dir}. 10, 30, 30, 20 and 10 clients have respectively 100, 250, 500, 750, and 1000 training samples, and testing samples amounting to a fifth of their training size.
The client local learning rate $\eta_l$ is selected in \{0.01, 0.02, 0.05, 0.1\}.

\begin{figure}[H]
    \centering
	\includegraphics[width = 0.35\textwidth]{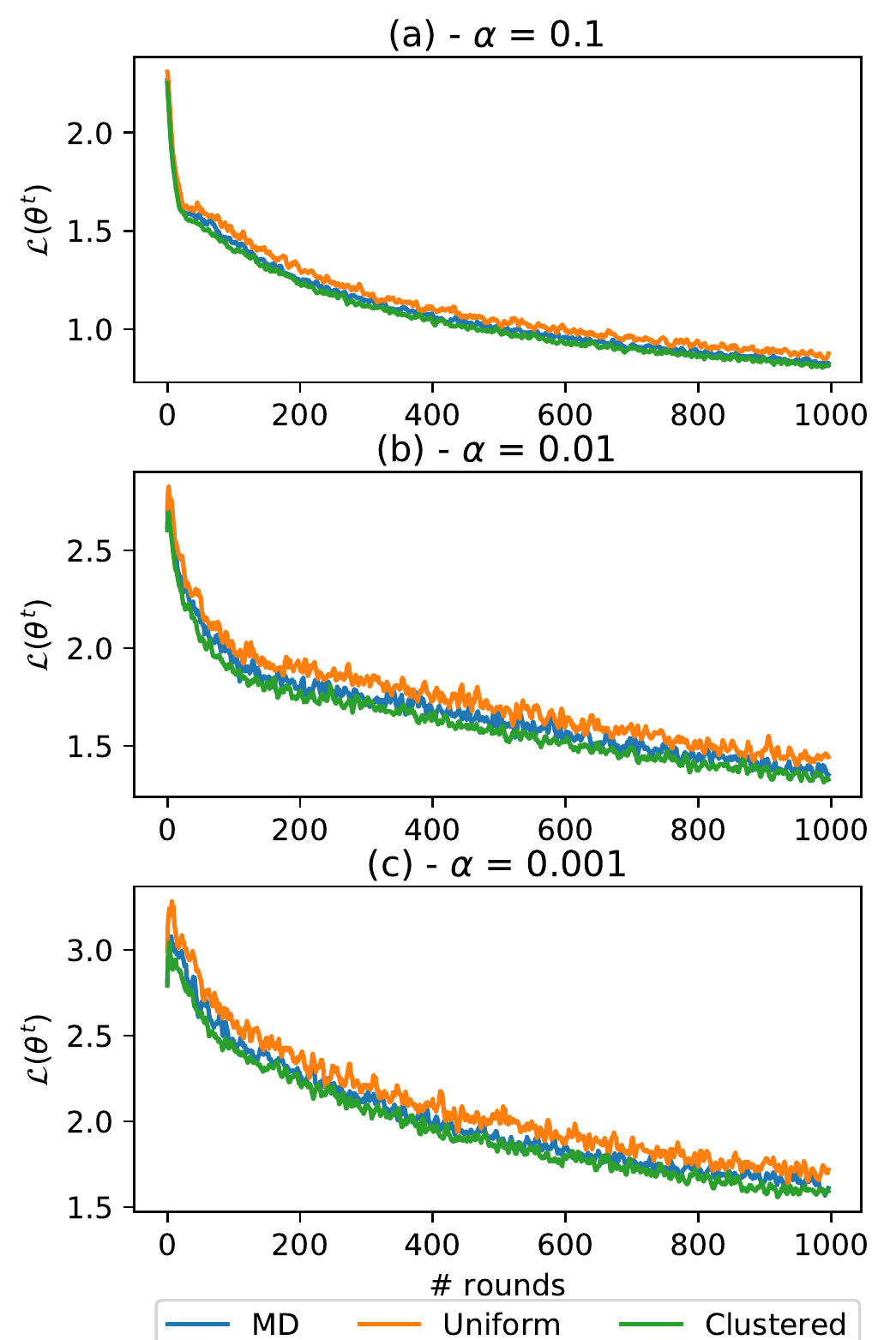}
	\caption{Convergence speed of the global loss with MD sampling and Uniform sampling when considering $n=100$ clients, while sampling $m = 10$ of them. Clients are partitioned using a Dirichlet distribution with parameter $\alpha = 0.1$ (a), $\alpha=0.01$ (b), and $\alpha = 0.001$ (c).
	Global losses are estimated on 30 different model initialization. }
	\label{fig:CIFAR_convergence}
\end{figure}

\end{document}